\documentclass[11pt]{article}

\usepackage{times}

\usepackage{parskip}

\renewcommand{\paragraph}[1]{
	\textbf{#1}~~}

 \usepackage[letterpaper, left=1in, right=1in, top=1in, bottom=1in]{geometry}

\usepackage[colorlinks=true, linkcolor=blue, citecolor=blue]{hyperref}
\usepackage{microtype}

\usepackage{natbib}
\bibpunct{(}{)}{;}{a}{,}{,}

\usepackage{mathtools}
\usepackage{amsmath}
\usepackage{bbm}
\usepackage{amsfonts}
\usepackage{amssymb}

\usepackage{MnSymbol} 
\usepackage{amsthm}
\usepackage{xpatch}

\newtheorem{lemma}{Lemma}
 
\newtheorem{theorem}{Theorem}

\xpatchcmd{\proof}{\itshape}{\normalfont\proofnameformat}{}{}
\newcommand{\proofnameformat}{\bfseries}

\usepackage{prettyref}
\newcommand{\pref}[1]{\prettyref{#1}}

\newcommand{\savehyperref}[2]{\texorpdfstring{\hyperref[#1]{#2}}{#2}}
\newrefformat{eq}{\savehyperref{#1}{\textup{(\ref*{#1})}}}
\newrefformat{eqn}{\savehyperref{#1}{Equation~\ref*{#1}}}
\newrefformat{lem}{\savehyperref{#1}{Lemma~\ref*{#1}}}
\newrefformat{def}{\savehyperref{#1}{Definition~\ref*{#1}}}
\newrefformat{thm}{\savehyperref{#1}{Theorem~\ref*{#1}}}
\newrefformat{fact}{\savehyperref{#1}{Fact~\ref*{#1}}}
\newrefformat{tab}{\savehyperref{#1}{Table~\ref*{#1}}}
\newrefformat{corr}{\savehyperref{#1}{Corollary~\ref*{#1}}}
\newrefformat{cor}{\savehyperref{#1}{Corollary~\ref*{#1}}}
\newrefformat{sec}{\savehyperref{#1}{Section~\ref*{#1}}}
\newrefformat{app}{\savehyperref{#1}{Appendix~\ref*{#1}}}
\newrefformat{ass}{\savehyperref{#1}{Assumption~\ref*{#1}}}
\newrefformat{ex}{\savehyperref{#1}{Example~\ref*{#1}}}
\newrefformat{fig}{\savehyperref{#1}{Figure~\ref*{#1}}}
\newrefformat{alg}{\savehyperref{#1}{Algorithm~\ref*{#1}}}
\newrefformat{rem}{\savehyperref{#1}{Remark~\ref*{#1}}}
\newrefformat{conj}{\savehyperref{#1}{Conjecture~\ref*{#1}}}
\newrefformat{prop}{\savehyperref{#1}{Proposition~\ref*{#1}}}
\newrefformat{proto}{\savehyperref{#1}{Protocol~\ref*{#1}}}
\newrefformat{prob}{\savehyperref{#1}{Problem~\ref*{#1}}}
\newrefformat{claim}{\savehyperref{#1}{Claim~\ref*{#1}}}
\newrefformat{tbl}{\savehyperref{#1}{Table~\ref*{#1}}}
\usepackage{algorithm}
\usepackage{algorithmic}

\DeclarePairedDelimiter{\abs}{\lvert}{\rvert} %
\DeclarePairedDelimiter{\brk}{[}{]}
\DeclarePairedDelimiter{\crl}{\{}{\}}
\DeclarePairedDelimiter{\prn}{(}{)}
\DeclarePairedDelimiter{\nrm}{\|}{\|}
\DeclarePairedDelimiter{\tri}{\langle}{\rangle}

\DeclarePairedDelimiter{\floor}{\lfloor}{\rfloor}

\let\Pr\undefined

\DeclareMathOperator{\En}{\mathbb{E}}

\DeclareMathOperator{\Pr}{Pr}

\DeclareMathOperator*{\argmin}{arg\,min} 

\newcommand{\eps}{\epsilon}

\newcommand{\ldef}{\vcentcolon=}

\newcommand{\wh}[1]{\widehat{#1}}

\def\ddefloop#1{\ifx\ddefloop#1\else\ddef{#1}\expandafter\ddefloop\fi}
\def\ddef#1{\expandafter\def\csname bb#1\endcsname{\ensuremath{\mathbb{#1}}}}
\ddefloop ABCDEFGHIJKLMNOPQRSTUVWXYZ\ddefloop
\def\ddefloop#1{\ifx\ddefloop#1\else\ddef{#1}\expandafter\ddefloop\fi}
\def\ddef#1{\expandafter\def\csname b#1\endcsname{\ensuremath{\mathbf{#1}}}}
\ddefloop ABCDEFGHIJKLMNOPQRSTUVWXYZ\ddefloop
\def\ddef#1{\expandafter\def\csname c#1\endcsname{\ensuremath{\mathcal{#1}}}}
\ddefloop ABCDEFGHIJKLMNOPQRSTUVWXYZ\ddefloop
\def\ddef#1{\expandafter\def\csname h#1\endcsname{\ensuremath{\widehat{#1}}}}
\ddefloop ABCDEFGHIJKLMNOPQRSTUVWXYZ\ddefloop
\def\ddef#1{\expandafter\def\csname hc#1\endcsname{\ensuremath{\widehat{\mathcal{#1}}}}}
\ddefloop ABCDEFGHIJKLMNOPQRSTUVWXYZ\ddefloop
\def\ddef#1{\expandafter\def\csname t#1\endcsname{\ensuremath{\widetilde{#1}}}}
\ddefloop ABCDEFGHIJKLMNOPQRSTUVWXYZ\ddefloop
\def\ddef#1{\expandafter\def\csname tc#1\endcsname{\ensuremath{\widetilde{\mathcal{#1}}}}}
\ddefloop ABCDEFGHIJKLMNOPQRSTUVWXYZ\ddefloop

\newcommand{\grad}{\nabla}

\usepackage{thmtools} 
\usepackage{thm-restate}
\usepackage{multirow}
\usepackage{mathrsfs}

\allowdisplaybreaks

\newcommand{\removed}[1]{}

\newcommand{\inner}[2]{\left\langle #1 , #2 \right\rangle}
\newcommand{\norm}[1]{\left\| #1 \right\|}
\newcommand{\parens}[1]{\left( #1 \right)}

\newcommand{\sfo}[0]{\lsO}
\newcommand{\dfo}[0]{\detO}
\newcommand{\sample}[0]{\gsO}

\newcommand{\mone}[0]{\mathsf{m}_\epsilon}

\newcommand{\FHL}[0]{\cF[H,\lambda]}

\newcommand{\FRHL}[0]{\dbF{H}{\lambda}{R}}
\newcommand{\FRHz}[0]{\dbF{H}{\lambda=0}{R}}

\newcommand{\FDHz}[0]{\rbF{H}{\lambda=0}{\Delta}}

\newcommand{\lsO}{\cO^\sigma _{\nabla f}}

\newcommand{\detO}{\cO_{\nabla F}}

\newcommand{\gsO}{\cO^\sigma_{f}}

\newcommand{\rbF}[3]{\cF_{\text{RB}}[#1, #2; #3]}
\newcommand{\dbF}[3]{\cF_{\text{DB}}[#1, #2; #3]}
\newcommand{\rbFD}[4]{\cF^{#4}_{\text{RB}}[#1, #2; #3]}
\newcommand{\dbFD}[4]{\cF^{#4}_{\text{DB}}[#1, #2; #3]}
\usepackage{array}
\newcolumntype{P}[1]{>{\centering\arraybackslash}p{#1}}

\newcommand{\expi}[1]{(#1)}
\newcommand{\empF}{\wh{F}_m}
\newcommand{\gltwo}{\textrm{AC-SA}^2}
\newcommand{\varsquared}{\sigma^{2}}

\title{The Complexity of Making the Gradient Small \\in Stochastic Convex Optimization}
\author{%
Dylan J. Foster\thanks{dylanf@mit.edu} \\
MIT\\
\and
Ayush Sekhari\thanks{sekhari@cs.cornell.edu}\\
Cornell University\\
\and 
Ohad Shamir\thanks{ohad.shamir@weizmann.ac.il}\\
Weizmann Institute\\
\and 
Nathan Srebro\thanks{nati@ttic.edu}\\
TTI Chicago\\
\and
Karthik Sridharan\thanks{sridharan@cs.cornell.edu}\\
Cornell University\\
\and
Blake Woodworth\thanks{blake@ttic.edu}\\
TTI Chicago\\
}
\date{}

\begin{document}
\maketitle 
\begin{abstract}%
We give nearly matching upper and lower bounds on the oracle complexity of finding $\eps$-stationary points $\prn*{\nrm{\nabla F(x)}\leq\epsilon}$ in stochastic convex optimization. We jointly analyze the oracle complexity in both the local stochastic oracle model and the global oracle (or, statistical learning) model. This allows us to decompose the complexity of finding near-stationary points into \emph{optimization complexity} and \emph{sample complexity}, and reveals some surprising differences between the complexity of stochastic optimization versus learning. Notably, we show that in the global oracle/statistical learning model, only \emph{logarithmic dependence on smoothness} is required to find a near-stationary point, whereas polynomial dependence on smoothness is necessary in the local stochastic oracle model. In other words, the separation in complexity between the two models can be exponential, and that the folklore understanding that smoothness is required to find stationary points is only weakly true for statistical learning.

Our upper bounds are based on extensions of a recent ``recursive regularization'' technique proposed by \citet{allen2018Howto}. We show how to extend the technique to achieve near-optimal rates, and in particular show how to leverage the extra information available in the global oracle model. Our algorithm for the global model can be implemented efficiently through finite sum methods, and suggests an interesting new computational-statistical tradeoff.

\end{abstract}

\section{Introduction}\label{sec:intro}

Success in convex optimization is typically defined as finding a point
whose value is close to the minimum possible value. \emph{Information-based complexity} of optimization attempts to understand the minimal amount of effort required to reach a desired level of suboptimality under different oracle models for access to the function \citep{nemirovskyyudin1983,traub1988information}. This complexity---for both deterministic and stochastic convex optimization---is tightly understood across a wide variety of settings \citep{nemirovskyyudin1983,traub1988information,agarwal2009information,braun2017lower}, and efficient algorithms that achieve optimal complexity are well known.

Recently, there has been a surge of interest in optimization for \emph{non-convex} functions. In this case, finding a point with near-optimal function value is typically intractable under standard assumptions---both computationally and information-theoretically. For this reason, a standard task in non-convex optimization is to find an $\eps$-stationary point, i.e., a point where the gradient is small $\prn*{\nrm{\grad{}F(x)}\leq{}\epsilon}$. 

In stochastic non-convex optimization, there has been a flurry of recent research on algorithms with provable guarantees for finding near-stationary points
\citep{ghadimi2013stochastic,DBLP:GhadimiL16,ReddiHSPS16,allen2017katyusha,lei2017non,jin2017escape,zhou2018stochastic,fang2018spider}. However, the stochastic oracle complexity of finding near-stationary points is not yet well understood, so we do not know whether existing algorithms
are optimal, or how we hope to improve upon them.

Recent work by \citet{carmon2017lower1,carmon2017lower2} establishes tight bounds on the {\em deterministic} first-order oracle complexity of finding near-stationary points of smooth functions, both convex and non-convex.  For convex problems, they prove that accelerated gradient descent is optimal both for finding approximate minimizers and approximate stationary points, while for non-convex problems, gradient descent is optimal for finding approximate stationary points.  The picture is simple and complete: the same deterministic first-order methods that are good at finding approximate minimizers are also good at finding approximate stationary points, even for non-convex functions. 

However, when one turns their attention to the {\em stochastic} oracle complexity of finding near-stationary points, the picture is far from clear. Even for \emph{stochastic convex optimization}, the oracle complexity is not yet well understood. This paper takes a first step toward resolving the general case by providing nearly tight upper and lower bounds on the oracle complexity of finding near-stationary points in stochastic convex optimization, both for first-order methods and for global (i.e., statistical learning) methods. At first glance, this might seem trivial, since exact minimizers are equivalent to exact stationary points for convex functions. When it comes to finding \emph{approximate} stationary points the situation is considerably more complex, and the equivalence does not yield optimal quantitative rates. For example, while the stochastic gradient descent (SGD) is (worst-case) optimal for stochastic convex optimization with a first-order oracle, it appears to be far from optimal for finding near-stationary points.

\begin{table}[t]
\center
\begin{tabular}{c c c c c}
&& Deterministic  & Sample & Stochastic  \\
&& First-Order Oracle & Complexity & First-Order Oracle \\\hline
\parbox[t]{1cm}{\multirow{2}{*}{\rotatebox[origin=c]{90}{\small{$\norm{x_0-x^*}\leq R$}}}}
& \small{Upper:}
& \parbox[t]{4cm}{\centering $\tilde{O}\prn*{\sqrt{\frac{HR}{\epsilon}}}$\\\small{\cite{nesterov2012make}}}
& \parbox[t]{3cm}{\centering $O\prn*{\frac{\sigma^2}{\epsilon^2}\log^{3}\parens{\frac{HR}{\epsilon}}}$\\\small{(\pref{cor:sample-complexity-bounded})}}
& \parbox[t]{3cm}{\centering $\tilde{O}\prn*{\sqrt{\frac{HR}{\epsilon}} + \frac{\sigma^2}{\epsilon^2}}$\\\small{(\pref{cor:sgd3-bounded-domain})}}
\rule[-20pt]{0pt}{37pt}\\ ~\\
& \small{Lower:}
& \parbox[t]{4cm}{\centering $\Omega\prn*{\sqrt{\frac{HR}{\epsilon}}}$\\\small{\cite{carmon2017lower2}}}
& \parbox[t]{3cm}{\centering $\Omega\prn*{\frac{\sigma^2}{\epsilon^2}}$\\\small{(\pref{thm:statistical-lower-bound})}}
& \parbox[t]{4cm}{\centering $\Omega\prn*{\sqrt{\frac{HR}{\epsilon}} + \frac{\sigma^2}{\epsilon^2}\log\parens{\frac{HR}{\epsilon}}}$\\\small{(\pref{thm:first-order-lower-bound})}}
\rule[-24pt]{0pt}{28pt}\\ \hline
\parbox[t]{1cm}{\multirow{2}{*}{\rotatebox[origin=c]{90}{\parbox[t]{2.345cm}{\centering \scriptsize{$F(x_0)-F(x^*)\leq \Delta$}}}}}
& \small{Upper:}
& \parbox[t]{4cm}{\centering $\tilde{O}\prn*{\frac{\sqrt{H\Delta}}{\epsilon}}$\\\small{\cite{carmon2017lower2}}}
& \parbox[t]{3cm}{\centering $O\prn*{\frac{\sigma^2}{\epsilon^2}\log^{3}\parens{\frac{H\Delta}{\epsilon^2}}}$\\\small{(\pref{cor:sample-complexity-bounded})}}
& \parbox[t]{3cm}{\centering $\tilde{O}\prn*{\frac{\sqrt{H\Delta}}{\epsilon} + \frac{\sigma^2}{\epsilon^2}}$\\\small{(\pref{cor:sgd3-bounded-domain})}}
\rule[-24pt]{0pt}{41pt}\\~\\
& \small{Lower:}
& \parbox[t]{4cm}{\centering $\Omega\prn*{\frac{\sqrt{H\Delta}}{\epsilon}}$\\\small{\cite{carmon2017lower2}}}
& \parbox[t]{3cm}{\centering $\Omega\prn*{\frac{\sigma^2}{\epsilon^2}}$\\\small{(\pref{thm:statistical-lower-bound})}}
& \parbox[t]{4cm}{\centering $\Omega\prn*{\frac{\sqrt{H\Delta}}{\epsilon} + \frac{\sigma^2}{\epsilon^2}\log\parens{\frac{H\Delta}{\epsilon^2}}}$\\\small{(\pref{thm:first-order-lower-bound})}}
\rule[-24pt]{0pt}{26pt}\\\hline
\end{tabular}
\caption{Upper and lower bounds on the complexity of finding $x$
  such that $\norm{\nabla F(x)}\leq\epsilon$ for convex problems with
  $H$-Lipschitz gradients, where $\sigma^2$ is a bound on the variance
  of gradient estimates.}
  \label{tab:results}
\end{table}

\subsection{Contributions}
We present a nearly tight analysis of the local stochastic oracle complexity and global stochastic oracle complexity (``sample complexity'') of finding approximate stationary points in stochastic convex optimization. Briefly, the highlights are as follows:
\begin{itemize}
\item We give upper and lower bounds on the local and global stochastic oracle complexity that match up to log factors. In particular, we show that the local stochastic complexity of finding stationary points is (up to log factors) characterized as the sum of the deterministic oracle complexity and the sample complexity.
\item As a consequence of this two-pronged approach, we show that the gap between local stochastic complexity and sample complexity of finding near-stationary points is at least \emph{exponential} in the smoothness parameter.
\item We obtain the above results through new algorithmic improvements. We show that the recursive regularization technique introduced by \cite{allen2018Howto} for local stochastic optimization can be combined with empirical risk minimization to obtain \emph{logarithmic} dependence on smoothness in the global model, and that the resulting algorithms can be implemented efficiently.
\end{itemize}

Complexity results are summarized in \pref{tab:results}. Here we discuss the conceptual contributions in more detail.

\paragraph{Decomposition of stochastic first-order complexity.}
For stochastic optimization of convex functions, there is a simple and powerful connection between three oracle complexities: Deterministic, local stochastic, and global stochastic. For many well-known problem classes, the stochastic
first-order complexity is equal to the sum (equivalently, maximum) of the
deterministic first-order complexity and the sample complexity.
This decomposition of the local stochastic complexity into an ``optimization term'' plus a ``statistical term'' inspires optimization methods, guides analysis, and facilitates comparison of different algorithms. It indicates that ``one pass'' stochastic approximation algorithms like SGD are optimal for stochastic optimization in certain parameter regimes, so that we do not have to resort to sample average approximation or methods that require multiple passes over data.

We establish that the same decomposition holds for the task of finding approximate stationary points. Such a characterization should not be taken for granted, and it is not clear a priori that it should hold
for finding stationary points. Establishing the result requires both developing new algorithms with near-optimal sample complexity in the global model, and improving previous local stochastic methods \citep{allen2018Howto} to match the optimal deterministic complexity. 

\paragraph{Gap between sample complexity and stochastic first-order complexity.}
For non-smooth convex objectives, finding an approximate stationary
point can require finding an \emph{exact} minimizer of the function (consider the absolute value function).
Therefore, as one would expect, the deterministic and stochastic first-order oracle complexities for finding near-stationary  points scale polynomially with the smoothness constant, even in low dimensions. Ensuring an approximate stationary point is impossible for non-smooth instances, even with an unbounded number of first-order oracle accesses.  Surprisingly, we show that the sample complexity depends at most logarithmically on the smoothness. In fact, in one dimension the dependence on smoothness can be removed entirely.

\paragraph{Improved methods.}
Our improved sample complexity results for the global stochastic oracle/statistical learning model are based on a new algorithm which uses the \emph{recursive regularization} (or, ``SGD3'') approach introduced by \cite{allen2018Howto}. The methods iteratively solves a sequence of subproblems via regularized empirical risk minimization (RERM). Solving subproblems through RERM allows the method to exploit global access to the stochastic samples. Since the method enjoys only logarithmic dependence on smoothness (as well as initial suboptimality or distance to the optimum), it provides a better alternative to \emph{any} stochastic first-order method whenever the smoothness is large relative to the variance in the gradient estimates. Since RERM is a finite-sum optimization problem, standard finite-sum optimization methods can be used to implement the method efficiently; the result is that we can beat the sample complexity of stochastic first-order methods with only modest computational overhead. 

For the local stochastic model, we improve the SGD3 method of \cite{allen2018Howto} so that the ``optimization'' term matches the optimal deterministic oracle complexity. This leads to a quadratic improvement in terms of the initial distance to the optimum (the ``radius'' of the problem), $\norm{x_0-x^*}$.  We also extend the analysis to the setting where initial sub-optimality $F(x_0) - F(x^*)$ is bounded but not the radius--a common setting in the analysis of non-convex optimization algorithms and a setting in which recursive regularization was not previously analyzed.  

\section{Setup}
We consider the problem of finding an $\epsilon-$stationary point in the stochastic convex optimization setting. That is, for a convex function $F: \bbR^d \mapsto \bbR$, our goal is to find a point $x \in \bbR^d$ such that
\begin{equation}
    \nrm*{\nabla F(x)} \leq \epsilon,
\end{equation}
given access to $F$ only through an \emph{oracle}.\footnote{Here, and for the rest of the paper, $\nrm*{\cdot}$ is taken to be the Euclidean norm.} Formally, the problem is specified by a class of functions to which $F$ belongs, and through the type of oracle through which we access $F$. We outline these now.

\paragraph{Function classes.}
Recall that $F: \bbR^d \to \bbR$ is is said to $H$-smooth if
  \begin{equation}
	F(y) \leq  F(x) + \tri*{\nabla F(x),  y-x} +  \frac{H}{2} \nrm*{y - x}^2\quad\forall x, y \in \bbR^d,
  \end{equation}
   and is said to be $\lambda$-strongly-convex if
  \begin{equation}
	F(y) \geq F(x) + \tri*{\nabla F(x), y - x} + \frac{\lambda}{2} \nrm*{y - x}^2\quad\forall x, y \in \bbR^d.
  \end{equation}

We focus on two classes of objectives, both of which are defined relative to an arbitrary initial point $x_0$ provided to the optimization algorithm.
\begin{enumerate}
\item \emph{Domain-bounded functions.}
\begin{equation}
\begin{aligned} 
\dbFD{H}{\lambda}{R}{d} &= \crl*{ F: \bbR^d \to \bbR ~~ \middle|  ~~ 
 \begin{array}{l} { \text{$F$ is $H$-smooth and $\lambda$-strongly convex  }}\\
 \text{$\argmin_{x}F(x)\neq{}\emptyset$} \\
\text{$\exists x^* \in \argmin_{x} F(x)$  s.t.  $\norm{x_0 - x^*} \leq R$} 
\end{array} 
}.
\end{aligned}
\end{equation}
\item \emph{Range-bounded functions.}
\begin{equation}
\begin{aligned}											       
\rbFD{H}{\lambda}{\Delta}{d}  &= \crl*{ F: \bbR^d \to \bbR ~~ \middle|  ~~  
\begin{array}{l} { \text{$F$ is $H$-smooth and $\lambda$-strongly convex  }  }\\
 \text{$\argmin_{x}F(x)\neq{}\emptyset$} \\
\text{$F(x_0) - \min_x F(x) \leq \Delta$ } 
\end{array} }.
\end{aligned}
\end{equation}
\end{enumerate}
We emphasize that while the classes are defined in terms of a strong convexity parameter, our main complexity results concern the non-strongly convex case where $\lambda=0$. The strongly convex classes are used for intermediate results. We also note that our main results hold in arbitrary dimension, and so we drop the superscript $d$ except when it is pertinent to discussion.

\paragraph{Oracle classes.}
An oracle accepts an argument $x \in \bbR^d$ and provides (possibly noisy/stochastic) information about the objective $F$ around the point $x$. The oracle's output belongs to an \emph{information space} $\cI$. We consider three distinct types of oracles:
  \begin{enumerate}
   \item \textbf{Deterministic first-order oracle.}  Denoted $\dfo$, with $\cI \subseteq \bbR^d \times (\bbR^d)^*$. When queried at a point $x \in \bbR^d$, the oracle returns 
   \begin{equation}
   \detO(x) = \prn*{F(x), \nabla F(x)}.
   \end{equation} 
    \item \textbf{Stochastic first-order oracle.} Denoted $\sfo$, with $\cI \subseteq \bbR^d \times \bbR^d$. The oracle is specified by a function $f:\bbR^d \times \cZ \rightarrow\bbR$ and a distribution $\cD$ over $\cZ$ with the property that $F(x) = \mathbb{E}_{z\sim\cD}\brk{f(x; z)}$ and $\sup_x \brk*{\mathbb{E}_{z\sim\cD}\norm{\nabla f(x; z) - F(x)}^2} \leq \sigma^2$. When queried at a point $x \in \bbR^d$, the oracle draws an independent $z \sim $ $\cD$ and returns
    \begin{equation}
        \sfo(x) = \prn*{f(x; z), \nabla f(x; z)}_{z \sim \cD}.
    \end{equation}
    \item \textbf{Stochastic global oracle.} Denoted $\sample$, with $\cI \subseteq (\bbR^d \mapsto \bbR)$. The oracle is specified by a function $f:\bbR^d \times \cZ \rightarrow\bbR$ and a distribution $\cD$ over $\cZ$ with the property that $F(x) = \mathbb{E}_{z\sim\cD}\brk{f(x; z)}$ and $\sup_x\brk*{\mathbb{E}_{z\sim\cD}\norm{\nabla f(x; z) - F(x)}^2} \leq \sigma^2$. When queried, the oracle draws an independent $z \in \cD$ and returns the complete specification of the function $f(\cdot, z)$, specifically, 
    \begin{equation}
        \sample(x) = \prn*{f(\cdot, z)}_{z \sim \cD}.
    \end{equation} 
    For consistency with the other oracles, we say that $\sample$ accepts an argument $x$, even though this argument is ignored. The global oracle captures the \emph{statistical learning} problem, in which $f(\cdot; z)$ is the loss of a model evaluated on an instance $z \sim \cD$, and this component function is fully known to the optimizer. Consequently, we use the terms ``global stochastic complexity'' and ``sample complexity'' interchangeably.

 \end{enumerate}
For the stochastic oracles, while $F$ itself may need to have properties such as convexity or smoothness, $f(\cdot;z)$ as defined need not have these properties unless stated otherwise.

\paragraph{Minimax oracle complexity.}
Given a function class $\cF$ and an oracle $\cO$ with information space $\cI$, we define the minimax oracle complexity of finding an $\epsilon$-stationary point as 
\begin{equation}
\mone(\cF,\cO) = \inf\crl*{m\in\bbN\ \middle|\ \inf_{A: \bigcup_{t \geq 0} \cI^t \mapsto \bbR^d}\ \sup_{F \in \cF}\ {\mathbb{E}}\norm{\nabla F(x_m)} \leq \epsilon },
\end{equation}
where $x_t \in \bbR^d$ is defined recursively as $x_t \ldef{} A(O(x_{0}), \ldots, O(x_{t-1}))  $ and the expectation is over the stochasticity of the oracle $\cO$.\footnote{See \pref{sec:first-order} for discussion of randomized algorithms.}

\paragraph{Recap: Deterministic first-order oracle complexity.}
To position our new results on stochastic optimization we must first recall what is known about the \emph{deterministic} first-order oracle complexity of finding near-stationary pointst. This complexity is tightly understood, with \[\mone(\FRHz,\dfo) = \tilde{\Theta}(\sqrt{HR}/\sqrt{\epsilon}),\quad\text{and}\quad\mone(\FDHz,\dfo) = \tilde{\Theta}\parens{\sqrt{H\Delta}/\epsilon},\] up to logarithmic factors \citep{nesterov2012make,carmon2017lower2}. The algorithm that achieves these rates is accelerated gradient descent (AGD). 

\section{Stochastic First-Order Complexity of Finding Stationary Points}\label{sec:first-order}
Interestingly, the usual variants of stochastic gradient descent do not appear to be optimal in the stochastic model. A first concern is that they do not yield the correct dependence on desired stationarity $\eps$.

As an illustrative example, let $F \in \FRHz$ and let any stochastic first-order oracle $\sfo$ be given. We adopt the naive approach of bounding stationarity by function value suboptimality. In this case the standard analysis of stochastic gradient descent (e.g., \cite{dekel2012optimal}) implies that after $m$ iterations, $\En\norm{\nabla F(x_m)} \leq O(\sqrt{H(\En{}F(x_m) - F(x^*))}) \leq O\parens{\sqrt{H\prn*{HR^2/m + \sigma R/\sqrt{m}}}}$, and thus \[\mone(\FRHz,\sfo) \leq O\parens{\frac{H^2R^2}{\epsilon^2} + \frac{H^2R^2\sigma^2}{\epsilon^4}}.\] The dependence on $\eps^{-4}$ is considerably worse than the $\eps^{-2}$ dependence enjoyed for function suboptimality.

In recent work, \citet{allen2018Howto} proposed a new \emph{recursive regularization} approach and used this in an algorithm called SGD3 that obtains the correct $\eps^{-2}$ dependence.\footnote{\citet{allen2018Howto} also show that some simple variants of SGD are able to reduce the poor $\epsilon^{-4}$ dependence to, e.g., $\eps^{-5/2}$, but they fall short of the $\epsilon^{-2}$ dependence one should hope for. Similar remarks apply for $F \in \FDHz$.}
 For any $F \in \FRHz$ and $\sfo$, SGD3 iteratively augments the objective with increasingly strong regularizers, ``zooming in'' on an approximate stationary point. Specifically, in the first iteration, SGD is used to find $\hat{x}_1$, an approximate minimizer of $F^{(0)}(x) = F(x)$. The objective is then augmented with a strongly-convex regularizer so $F^{(1)}(x) = F^{(0)}(x) + \lambda\norm{x-\hat{x}_1}^2$. In the second round, SGD is initialized at $\hat{x}_1$, and used to find $\hat{x}_2$, an approximate minimizer of $F^{(1)}$. This process is repeated, with $F^{(t)}(x) \ldef F^{(t-1)}(x) + 2^{t-1}\lambda\norm{x-\hat{x}_t}^2$ for each $t \in [T]$. \cite{allen2018Howto} shows that SGD3 find an $\eps$-stationary points using at most
\begin{equation}\label{eq:sgd3-ub}
   m \leq \tilde{O}\parens{\frac{HR}{\epsilon} + \frac{\sigma^2}{\epsilon^2}}
\end{equation}
local stochastic oracle queries. This oracle complexity has a familiar structure: it resembles the sum of an ``optimization term'' ($HR/\eps$) and a ``statistical term'' ($\sigma^{2}/\eps^{2}$). While we show that the statistical term is tight up to logarithmic factors (\pref{thm:first-order-lower-bound}), the optimization term does not match the $\Omega(\sqrt{HR / \epsilon})$ lower bound for the deterministic setting \citep{carmon2017lower2}.

\begin{algorithm}
\caption{Recursive Regularization Meta-Algorithm}
\begin{algorithmic} 
\label{alg:meta-algorithm}
\REQUIRE A function $F \in \FHL$, an oracle $\cO$ and an alloted number of oracle accesses $m$, an initial point $x_0$, and an optimization sub-routine $\cA$, with  $A = \cA[\cO,m/\floor{\log_2\frac{H}{\lambda}}]$.
\STATE $F^{(0)} \ldef{} F$, $\hat{x}_0\ldef{}x_0$, $T \leftarrow \floor{\log_2\frac{H}{\lambda}}$.
\FOR {$t = 1 \text{ to } T$}
\STATE $\hat{x}_t$ is output of $A$ used to optimize $F^{(t-1)}$ intitialized at $\hat{x}_{t-1}$
\STATE $F^{(t)}(x) \ldef{}  F(x) + \lambda\sum_{k=1}^{t} 2^{k-1} \nrm*{x - \hat{x}_k}^2$
\ENDFOR 
\RETURN $\hat{x}_T$
\end{algorithmic}
\end{algorithm}
Our first result is to close this gap. The key idea is to view SGD3 as a template algorithm, where the inner loop of SGD used in \cite{allen2018Howto} can be swapped out for an arbitrary optimization method $\cA$. This template, \pref{alg:meta-algorithm}, forms the basis for all the new methods in this paper.\footnote{The idea of replacing the sub-algorithm in SGD3 was also used by \cite{davis2018complexity}, who showed that recursive regularization with a projected subgradient method can be used to find near-stationary points for the Moreau envelope of any Lipschitz function.
} 

To obtain optimal complexity for the local stochastic oracle model we use a variant of the accelerated stochastic approximation method (``AC-SA'') due to \cite{ghadimi2012optimal} as the subroutine. Pseudocode for AC-SA is provided in \pref{alg:ac-sa}. We use a variant called $\gltwo$, see \pref{alg:gltwo}. The $\gltwo$ algorithm is equivalent to AC-SA, except the stepsize parameter is reset halfway through. This leads to slightly different dependence on the smoothness and domain size parameters, which is important to control the final rate when invoked within \pref{alg:meta-algorithm}. 

Toward proving the tight upper bound in \pref{tab:results}, we first show that \pref{alg:meta-algorithm} with $\gltwo$ as its subroutine guarantees fast convergence for strongly-convex domain-bounded objectives. 
\begin{restatable}{theorem}{accsgdthreesc}\label{thm:acc-sgd3-strongly-convex}
For any $F \in \FRHL$ and any $\sfo$, \pref{alg:meta-algorithm} using $\gltwo$ as its subroutine finds a point $\hat{x}$ with $\mathbb{E}\norm{\nabla F(\hat{x})} \leq \epsilon$ using 

\[
m \leq O\parens{\sqrt{\frac{H}{\lambda}}\log\parens{\frac{H}{\lambda}} + \sqrt{\frac{HR}{\epsilon}}\log\parens{\frac{H}{\lambda}} + \parens{\frac{\sqrt{H}\sigma}{\sqrt{\lambda}\epsilon}}^{\frac{2}{3}}\log\parens{\frac{H}{\lambda}} + \frac{\sigma^2}{\epsilon^2}\log^3\parens{\frac{H}{\lambda}}}
\]
total stochastic first-order oracle accesses. 
\end{restatable}

The analysis of this algorithm is detailed in \pref{app:proofs-first-order} and carefully matches the original analysis of SGD3 \citep{allen2018Howto}. The essential component of the analysis is \pref{lem:zeyuan_auxillary_lemma}, which provides a bound on $\norm{\nabla F(\hat{x})}$ in terms of the optimization error of each invocation of $\gltwo$ on the increasingly strongly convex subproblems $F^{(t)}$. 

\begin{algorithm}[t]
\caption{AC-SA}\label{alg:ac-sa}
\begin{algorithmic}  
\REQUIRE A function $F \in \FRHL$, a stochastic first-order oracle $\sfo$, and an alloted number of oracle accesses $m$
\STATE $x_0^{ag} = x_0$
\FOR{$t=1,2,\dots,m$}
\STATE $\alpha_t \leftarrow \frac{2}{t+1}$ \\
\STATE $\gamma_t \leftarrow \frac{4H}{t(t+1)}$
\STATE $x_t^{md} \leftarrow \frac{(1-\alpha_t)(\lambda+\gamma_t)}{\gamma_t + (1-\alpha_t^2)\lambda}x_{t-1}^{ag} + \frac{\alpha_t\parens{(1-\alpha_t)\lambda + \gamma_t}}{\gamma_t + (1-\alpha_t^2)\lambda}x_{t-1}$
\STATE $\nabla f(x^{md}_t;z_t) \leftarrow \sfo(x_t^{md})$
\STATE $x_t \leftarrow \frac{\alpha_t\lambda}{\lambda+\gamma_t}x_t^{md} + \frac{(1-\alpha_t)\lambda + \gamma_t}{\lambda+\gamma_t}x_{t-1} - \frac{\alpha_t}{\lambda+\gamma_t}\nabla f(x^{md}_t;z_t)$
\STATE $x_t^{ag} \leftarrow \alpha_t x_t + (1-\alpha_t)x_{t-1}^{ag}$
\ENDFOR
\RETURN $x_m^{ag}$
\end{algorithmic}
\end{algorithm}

Our final result for non-strongly convex objectives uses \pref{alg:meta-algorithm} with $\gltwo$ on the regularized objective $\tilde{F}(x) = F(x) + \frac{\lambda}{2}\norm{x-x_0}^2$. The performance guarantee is as follows, and concerns both domain-bounded and range-bounded functions.
\begin{restatable}{corol}{accsgdthreeboundeddomain}\label{cor:sgd3-bounded-domain}
For any $F \in \FRHz$ and any $\sfo$, \pref{alg:meta-algorithm} with $\gltwo$ as its subroutine applied to $F(x) + \frac{\lambda}{2}\norm{x-x_0}^2$ for $\lambda = \Theta\parens{\min\crl*{\frac{\epsilon}{R}, \frac{H\epsilon^4}{\sigma^4\log^4(\sigma/\epsilon)}}}$ yields a point $\hat{x}$ such that $\mathbb{E}\norm{\nabla F(\hat{x})} \leq \epsilon$ using 
\[
m \leq O\parens{\sqrt{\frac{HR}{\epsilon}}\log\parens{\frac{HR}{\epsilon}} + \frac{\sigma^2}{\epsilon^2}\log^3\parens{\frac{\sigma}{\epsilon}}}
\]
total stochastic first-order oracle accesses. \\
For any $F \in \FDHz$ and any $\sfo$, the same algorithm with $\lambda = \Theta\parens{\min\crl*{\frac{\epsilon^2}{\Delta}, \frac{H\epsilon^4}{\sigma^4\log^4(\sigma/\epsilon)}}}$ yields a point $\hat{x}$ with $\mathbb{E}\norm{\nabla F(\hat{x})} \leq \epsilon$ using 
\[
m \leq O\parens{\frac{\sqrt{H\Delta}}{\epsilon}\log\parens{\frac{\sqrt{H\Delta}}{\epsilon}} + \frac{\sigma^2}{\epsilon^2}\log^3\parens{\frac{\sigma}{\epsilon}}}
\]
total stochastic first-order oracle accesses.
\end{restatable}

This follows easily from \pref{thm:acc-sgd3-strongly-convex} and is proven in \pref{app:proofs-first-order}. Intuitively, when $\lambda$ is chosen appropriately, the gradient of the regularized objective $\tilde{F}$ does not significantly deviate from the gradient of $F$, but the number of iterations required to find an $O(\epsilon)$-stationary point of $\tilde{F}$ is still controlled. 

We now provide nearly-tight lower bounds for the stochastic first-order oracle complexity. A notable feature of the lower bound is to show that show some of the logarithmic terms in the upper bound---which are not present in the optimal oracle complexity for function value suboptimality---are necessary.
\begin{restatable}{theorem}{folowerbound}\label{thm:first-order-lower-bound}
For any $H,\Delta,R,\sigma > 0$, any $\epsilon \leq \frac{HR}{8}$, the stochastic first-order oracle complexity for range-bounded functions is lower bounded as
\[
    \mone(\FRHz,\sfo) \geq \Omega\parens{\sqrt{\frac{HR}{\epsilon}} + \frac{\sigma^2}{\epsilon^2}\log\parens{\frac{HR}{\epsilon}}}.
\]
For any $\epsilon \leq \sqrt{\frac{H\Delta}{8}}$, the stochastic first-order complexity for domain-bounded functions is lower bounded as
\[
    \mone(\FDHz,\sfo) \geq \Omega\parens{\frac{\sqrt{H\Delta}}{\epsilon} + \frac{\sigma^2}{\epsilon^2}\log\parens{\frac{H\Delta}{\epsilon^2}}}.
\]
\end{restatable}
The proof, detailed in \pref{app:proofs-lower-bounds}, combines the existing lower bound on the deterministic first-order oracle complexity \citep{carmon2017lower2} with a new lower bound for the statistical term. The approach is to show that any algorithm for finding near-stationary points can be used to solve noisy binary search (NBS), and then apply a known lower bound for NBS \citep{feige1994computing, karp2007noisy}. It is possible to extend the lower bound to randomized algorithms; see discussion in \cite{carmon2017lower2}.

\begin{algorithm}[t]
\caption{$\gltwo$}\label{alg:gltwo}
\begin{algorithmic} 
\REQUIRE A function $F \in \FRHL$, a stochastic first-order oracle $\sfo$, and an alloted number of oracle accesses $m$
\STATE $x_1 \leftarrow$ AC-SA$\prn*{F, x_0, \frac{m}{2}}$
\STATE $x_2 \leftarrow$ AC-SA$\prn*{F, x_1, \frac{m}{2}}$
\RETURN $x_2$
\end{algorithmic}
\end{algorithm}

\section{Sample Complexity of Finding Stationary Points}\label{sec:sample-complexity}
Having tightly bound the stochastic first-order oracle complexity of finding approximate stationary points, we now turn to sample complexity. If the heuristic reasoning that stochastic first-order complexity should decompose into sample complexity and deterministic first-order complexity ($\smash{\mone(\cF,\sfo)} \approx \smash{\mone(\cF,\dfo)} + \mone(\cF,\sample)$) is correct, then one would expect that the sample complexity should be $\tilde{O}(\sigma^2/\epsilon^2)$ for both domain-bounded and range-bounded function.

A curious feature of this putative sample complexity is that it does not depend on the smoothness of the function. This is somewhat surprising since if the function is non-smooth in the vicinity of its minimizer, there may only be a single $\epsilon$-stationary point, and an algorithm would need to return \emph{exactly} that point using only a finite sample. We show that the sample complexity is in fact almost independent of the smoothness constant, with a mild logarithmic dependence. We also provide nearly tight lower bounds. 

For the global setting, a natural algorithm to try is regularized empirical risk minimization (RERM), which returns $\hat{x} = \argmin_x \frac{1}{m}\sum_{i=1}^mf(x;z_i) + \frac{\lambda}{2}\norm{x-x_0}^2$.\footnote{While it is also tempting to try constrained ERM, this does not succeed even for function value suboptimality \citep{shalev2009stochastic}.} For any domain-bounded function $F \in \FRHz$, a standard analysis of ERM based on stability \citep{shalev2009stochastic} shows that $\mathbb{E}\norm{\nabla F(\hat{x})} \leq \mathbb{E}\sqrt{2H(F(\hat{x}) - F^*)} + \lambda R \leq O(\sqrt{H^{3}R^{2}/\lambda m} + \lambda R)$. Choosing $m = \Omega((HR)^3/\epsilon^3)$ and $\lambda = \Theta(\epsilon/R)$ yields an $\epsilon$-stationary point. This upper bound, however,  has two shortcomings.  First, it scales with $\smash{\epsilon^{-3}}$ rather than $\smash{\epsilon^{-2}}$ that we hoped for and, second,  it does not approach $1$ as $\sigma \rightarrow 0$, which one should expect in the noise-free case. The stochastic first-order algorithm from the previous section has better sample complexity, but the number of samples still does not approach one when $\sigma\rightarrow{}0$.

We fix both issues by combining regularized ERM with the recursive regularization approach, giving an upper bound that nearly matches the sample complexity lower bound $\Omega(\sigma^2/\epsilon^2)$. They key tool here is a sharp analysis of regularized ERM---stated in the appendix as \pref{thm:erm_variance}---that obtains the correct dependence on the variance $\sigma^{2}$.

As in the previous section, we first prove an intermediate result for the strongly convex case. Unlike \pref{sec:first-order}, where $F$ was required to be convex but the components $f(\cdot;z)$ were not required to be, we must assume here either that $f(\cdot;z)$ is convex for all $z$.\footnote{We are not aware of any analysis of ERM for strongly convex losses that does not make such an assumption. It is interesting to see whether this can be removed.} 
\begin{restatable}{theorem}{samplecomplexitySC}\label{thm:sample-complexity-strong-convexity}
For any $F \in \FHL$ and any global stochastic oracle $\sample$ with the restriction that $f(\cdot;z)$ is convex for all $z$, \pref{alg:meta-algorithm} with ERM as its subroutine finds $\hat{x}$ with $\mathbb{E}\norm{\nabla F(\hat{x})} \leq \epsilon$ using at most
\[
m \leq O\parens{\frac{\sigma^2}{\epsilon^2}\log^3\parens{\frac{H}{\lambda}}}
\]
total samples.
\end{restatable}
The proof is given in \pref{app:proofs-sample-complexity}. As before, we handle the non-strongly convex case by applying the algorithm to $\tilde{F}(x) = F(x) + \frac{\lambda}{2}\norm{x-x_0}^2$.
\begin{restatable}{corol}{samplecomplexitybounded}\label{cor:sample-complexity-bounded}
For any $F \in \FRHz$ and any global stochastic oracle $\sample$ with the restriction that $f(\cdot;z)$ is convex for all $z$, \pref{alg:meta-algorithm} with ERM as its subroutine, when applied to $\tilde{F}(x) = F(x) + \frac{\lambda}{2}\norm{x-x_0}^2$ with $\lambda = \Theta(\epsilon/R)$, finds a point $\hat{x}$ with $\mathbb{E}\norm{\nabla F(\hat{x})}\leq\epsilon$ using at most
\[
m \leq O\parens{\frac{\sigma^2}{\epsilon^2}\log^3\parens{\frac{HR}{\epsilon}}}
\]
total samples. \\
For any $F \in \FDHz$ and any global stochastic oracle $\sample$ with the restriction that $f(\cdot;z)$ is convex for all $z$, the same approach with $\lambda = \Theta(\epsilon^2/\Delta)$ finds an $\eps$-stationary point using at most
\[
m \leq O\parens{\frac{\sigma^2}{\epsilon^2}\log^3\parens{\frac{\sqrt{H\Delta}}{\epsilon}}}
\]
total samples.
\end{restatable}
This follows immediately from \pref{thm:sample-complexity-strong-convexity} by choosing $\lambda$ small enough such that $\norm{\nabla F(x)} \approx \norm{\nabla \tilde{F}(x)}$. Details are deferred to \pref{app:proofs-sample-complexity}. 

With this new sample complexity upper bound, we proceed to provide an almost-tight lower bound.


\begin{restatable}{theorem}{statlowerbound}\label{thm:statistical-lower-bound}
For any $H,\Delta,R,\sigma > 0$, $\epsilon \leq \min\crl{\frac{HR}{8}, \sqrt{\frac{H\Delta}{8}}, \frac{\sigma}{4}}$, the sample complexity to find a $\eps$-stationary point\footnote{This lower bound applies both to deterministic and randomized optimization algorithms.} is lower bounded as
\[
    \mone(\FRHz \cap \FDHz,\gsO) \geq \Omega\parens{\frac{\sigma^2}{\epsilon^2}}.
\]
\end{restatable}
This lower bound is similar to constructions used to prove lower bounds in the case of finding an approximate minimizer \citep{nemirovskyyudin1983,nesterov2004introductory,woodworth16tight}. However, our lower bound applies for functions with simultaneously bounded domain and range, so extra care must be taken to ensure that these properties hold. The lower bound also ensures that $f(\cdot;z)$ is convex for all $z$. The proof is located in \pref{app:proofs-lower-bounds}.

\paragraph{Discussion: Efficient implementation.}
\pref{cor:sample-complexity-bounded} provides a bound on the number of \emph{samples} needed to find a near-stationary point. However, a convenient property of the method is that the ERM objective $F^{(t)}$ solved in each iteration is convex, $(H+2^t\lambda)$-smooth, $(2^t\lambda)$-strongly convex, and has \emph{finite sum} structure with $m/T$ components. These subproblems can therefore be solved using at most $O\parens{\parens{\frac{m}{T} + \sqrt{\frac{m(H+\lambda 2^t)}{T\lambda 2^t}}}\log\frac{HR}{\epsilon}}$ gradient computations via a first-order optimization algorithm such as Katyusha \citep{allen2017katyusha}. This implies that the method can be implemented with a total gradient complexity  of $O\parens{\parens{\frac{\sigma^2}{\epsilon^2} + \frac{\sigma^{3/2}\sqrt{H}}{\epsilon^{3/2}}}\log^4\parens{\frac{HR}{\epsilon}}}$ over all $T$ iterations, and similarly for the bounded-range case. Thus, the algorithm is not just sample-efficient, but also computationally efficient, albeit slightly less so than the algorithm from \pref{sec:first-order}.

\paragraph{Removing smoothness entirely in one dimension.}
The gap between the upper and lower bounds for the statistical complexity is quite interesting. We conclude from \pref{cor:sample-complexity-bounded} that the sample complexity depends at most logarithmically upon the smoothness constant, which raises the question of whether it must depend on the smoothness at all. We now show that for the special case of functions in one dimension, smoothness is not necessary. In other words, all that is required to find an $\eps$-stationary point is Lipschitzness.

\begin{restatable}{theorem}{samplecomplexitynonsmooth}\label{thm:sample-complexity-non-smooth}
Consider any convex, $L$-Lipschitz function $F:\mathbb{R}\rightarrow\mathbb{R}$ that is bounded from below,\footnote{This lower bound does not enter the sample complexity quantitatively.} and any global stochastic oracle $\sample$ with the restriction that $f(\cdot;z)$ is convex for all $z$. There exists an algorithm which uses $m = O\parens{\frac{\sigma^2\log\parens{\frac{L}{\epsilon}}}{\epsilon^2}}$ samples and outputs a point $\hat{x}$ such that $\mathbb{E}\brk*{\inf_{g\in\partial F(\hat{x})} ~\abs{g}} \leq \epsilon$. 
\end{restatable}
The algorithm calculates the empirical risk minimizer on several independent samples, and then returns the point that has the smallest empirical gradient norm on a validation sample. The proof uses the fact that any function $F$ as in the theorem statement has a single left-most and a single right-most $\epsilon$-stationary point. As long as the empirical function's derivative is close to $F$'s at those two points, we argue that the ERM lies between them with constant probability, and is thus an $\epsilon$-stationary point of $F$. We are able to boost the confidence by repeating this a logarithmic number of times. A rigorous argument is included in \pref{app:proofs-sample-complexity}. Unfortunately, arguments of this type does not appear to extend to more than one dimension, as the boundary of the set of $\epsilon$-stationary points will generally be uncountable, and thus it is not apparent that the empirical gradient will be uniformly close to the population gradient. It remains open whether smoothness is needed in two dimensions or more.

The algorithm succeeds even for non-differentiable functions, and requires neither strong convexity nor knowledge of a point $x_0$ for which $\norm{x_0-x^*}$ or $F(x_0) - F^*$ is bounded. In fact, the assumption of Lipschitzness (more generally, $L$-subgaussianity of the gradients) is only required to get an in-expectation statement. Without this assumption, it can still be shown that ERM finds an $\epsilon$-stationary point with constant probability using $m \leq O\prn*{\frac{\sigma^2}{\epsilon^2}}$ samples. 
%
%

\section{Discussion}\label{sec:discussion}
We have proven nearly tight bounds on the oracle complexity of finding near-stationary points in stochastic convex optimization, both for local stochastic oracles and global stochastic oracles. We hope that the approach of jointly studying stochastic first-order complexity and sample complexity will find use more broadly in non-convex optimization. To this end, we close with a few remarks and open questions.

\begin{enumerate}
\item \textit{Is smoothness necessary for finding $\epsilon$-stationary points?} 
While the logarithmic factor separating the upper and lower bound we provide for stochastic first-order oracle complexity is fairly inconsequential, the gap between the upper and lower bound on the \emph{sample complexity} is quite interesting. In particular, we show through \pref{thm:statistical-lower-bound} and \pref{cor:sample-complexity-bounded} that
\begin{equation*}
\Omega\parens{\frac{\sigma^2}{\epsilon^2}} \leq \mone\parens{\FDHz, \gsO} \leq O\parens{\frac{\sigma^2}{\epsilon^2}\log^3\parens{\frac{\sqrt{H\Delta}}{\epsilon}}},
\end{equation*}
and similarly for the domain-bounded case. Can the $\mathrm{polylog}(H)$ factor on the right-hand side be removed entirely? Or in other words, is it possible to find near-stationary points in the statistical learning model without smoothness?\footnote{For a general non-smooth function $F$, a point $x$ is said to be an $\epsilon$-stationary point if there exists $v \in \partial F(x)$ such that $\nrm*{v}_2 \leq \epsilon$.} By \pref{thm:sample-complexity-non-smooth}, we know that this is possible in one dimension.


\item \textit{Tradeoff between computational complexity and sample complexity.} Suppose our end goal is to find a near-stationary point in the statistical learning setting, but we wish to do so efficiently. For range-bounded functions, if we use \pref{alg:meta-algorithm} with $\gltwo$ as a subroutine we require $\tilde{O}\prn*{\frac{\sqrt{H\Delta}}{\epsilon} + \frac{\sigma^2}{\epsilon^2} }$ samples, and the total computational effort (measured by number of gradient operations) is also $\tilde{O}\prn*{\frac{\sqrt{H\Delta}}{\epsilon} + \frac{\sigma^2}{\epsilon^2} }$. On the other hand, if we use \pref{alg:meta-algorithm} with RERM as a subroutine and implement RERM with Katyusha, then we obtain an improved sample complexity of $\tilde{O}\prn*{\frac{\sigma^2}{\epsilon^2}}$, but at the cost of a larger number of gradient operations: $\tilde{O}\prn*{ \frac{\sigma^2}{\epsilon^2} + \frac{\sqrt{H}\sigma^{3/2}}{\epsilon^{3/2}}}$. Tightly characterizing such computational-statistical tradeoffs in this and related settings is an interesting direction for future work.

\item \textit{Active stochastic oracle.} For certain stochastic first-order optimization algorithms based on variance reduction (SCSG \citep{lei2017non}, SPIDER \citep{fang2018spider}), a gradient must be computed at multiple points for the same sample $f(\cdot;z)$. We refer to such algorithms as using an ``active query'' first-order stochastic oracle, which is a stronger oracle than the classical first-order stochastic oracle (see \cite{woodworth2018graph} for more discussion). It would be useful to characterize the exact oracle complexity in this model, and in particular to understand how many active queries are required to obtain logarithmic dependence on smoothness as in the global case.


\item \textit{Complexity of finding stationary points for smooth non-convex functions.}
An important open problem is to characterize the minimax oracle complexity of finding near-stationary points for smooth non-convex functions, both for local and global stochastic oracles. For a deterministic first-order oracle, the optimal rate is $\tilde{\Theta}\prn*{\frac{H\Delta}{\epsilon^2}}$. In the stochastic setting, a simple sample complexity lower bound follows from the convex case, but this is not known to be tight. 

\end{enumerate}
 
\paragraph{Acknowledgements}
We would like to thank Srinadh Bhojanapalli and Robert D. Kleinberg for helpful discussions. Part of this work was completed while DF was at Cornell University and supported by the Facebook Ph.D. fellowship. OS is partially supported by a European Research Council (ERC) grant. OS and NS are partially supported by an NSF/BSF grant. BW is supported by the NSF Graduate Research Fellowship under award 1754881.

\bibliography{refs}

\appendix
\section{Proofs from \pref{sec:first-order}: Upper Bounds}\label{app:proofs-first-order}

\begin{theorem}[Proposition 9 of \cite{ghadimi2012optimal}]
\label{thm:ghadimi_lan_AC_SA}
For any $F \in \FRHL$ and any $\sfo$, the AC-SA algorithm returns a point $\hat{x}_T$ after making $T$ oracle accesses such that
\[
\En\brk*{F\prn*{\hat{x}_T}} - F(x^*)  \leq 
\frac{2HR^{2}}{T^{2}}
+ \frac{8\varsquared}{\lambda T}.
\]
\end{theorem}

\begin{lemma}
\label{lem:gd2}
For any $F \in \FRHL$ and any $\sfo$, the $\gltwo$ algorithm returns a point $\hat{x}$ after making $T$ oracle accesses such that
\[
\En\brk*{F\prn*{\hat{x}}} - F(x^*)  \leq 
\frac{128H^2R^{2}}{\lambda T^{4}} + \frac{256H\varsquared}{\lambda^2{}T^3} + \frac{16\varsquared}{\lambda{}T}.
\]
\end{lemma}
\begin{proof}
By \pref{thm:ghadimi_lan_AC_SA}, the first instance of AC-SA outputs $\hat{x}_1$ such that
\begin{equation}
\En\brk*{F\prn*{\hat{x}_1}} - F(x^*)  \leq 
\frac{8HR^{2}}{T^{2}}
+ \frac{16\varsquared}{\lambda{}T},
\end{equation}
and since $F$ is $\lambda$-strongly convex, 
\begin{equation}
\frac{\lambda}{2}\En\nrm*{\hat{x}_1-x^*}^{2}\leq 
\En\brk*{F\prn*{\hat{x}_1}} - F(x^*)
\leq
\frac{8HR^{2}}{T^{2}}
+ \frac{16\varsquared}{\lambda{}T}.
\end{equation}
Also by \pref{thm:ghadimi_lan_AC_SA}, the second instance of AC-SA outputs $\hat{x}_2$ such that
\begin{align}
\En\brk*{F\prn*{\hat{x}_2} - F(x^*)} 
&= \En\brk*{\En\left[ F\prn*{\hat{x}_2} - F(x^*)\ \middle|\ \hat{x}_1 \right]} \\
&\leq \En\brk*{\frac{8H\nrm*{\hat{x}_{1} - x^*}^{2}}{T^{2}} + \frac{16\varsquared}{\lambda{}T}} \\
&\leq \frac{128H^2R^{2}}{\lambda T^{4}} + \frac{256H\varsquared}{\lambda^2{}T^3} + \frac{16\varsquared}{\lambda{}T}.
\end{align}
\end{proof}

\begin{lemma}[Claim 6.2 of \cite{allen2018Howto}] \label{lem:zeyuan_auxillary_lemma}
Suppose that for every $t = 1, \ldots, T$ the iterates of \pref{alg:meta-algorithm} satisfy
$ 
\En \brk*{F^{(t-1)}(\hat{x}_t)} - F^{(t-1)}(x^*_{t-1}) \leq \delta_t
$ where $x^*_{t-1} = \argmin_x F^{(t-1)}(x)$,
then
\begin{enumerate}
\item For all $t \geq 1$, $\En \brk*{ \nrm*{\hat{x}_t - x^*_{t-1}}}^2 \leq \En \brk*{  \nrm*{\hat{x}_t - x^*_{t-1}}^2 } \leq \dfrac{\delta_t}{2^{t-2}\lambda}$.
\item For every $t \geq 1$, $\En \brk*{\nrm*{\hat{x}_t - x_t^*}}^2 \leq \En \brk*{\nrm*{\hat{x}_t - x_t^*}^2} \leq \dfrac{\delta_t}{2^t\lambda}$.
\item For all $t \geq 1$, $\En \brk*{\sum_{t=1}^T 2^t \lambda \nrm*{\hat{x}_t - x_T^*}} \leq 4 \sum_{t=1}^T \sqrt{2^t\lambda\delta_t }.$
\end{enumerate}
\end{lemma} ~

\accsgdthreesc*
\begin{proof}
As in \pref{lem:zeyuan_auxillary_lemma}, let $\mathbb{E}\brk{F^{(t-1)}(\hat{x}_t) - F^{(t-1)}(x^*_{t-1})} \leq \delta_t$ for each $t \geq 1$. The objective in the final iteration, $F^{(T-1)}(x) = F(x) + \lambda\sum_{t=1}^{T-1}2^{t-1}\norm{x-\hat{x}_t}^2$, so
\begin{align}
\norm{\nabla F(\hat{x}_T)} 
&= \norm{\nabla F^{(T-1)}(\hat{x}_T) + \lambda\sum_{t=1}^{T-1}2^{t}(\hat{x}_t - \hat{x}_T)} \\
&\leq \norm{\nabla F^{(T-1)}(\hat{x}_T)} + \lambda\sum_{t=1}^{T-1}2^{t}\norm{\hat{x}_t - \hat{x}_T} \label{eq:thm1-triangle-ineq-1}\\
&\leq \norm{\nabla F^{(T-1)}(\hat{x}_T)} + \lambda\sum_{t=1}^{T-1}2^{t}\parens{\norm{\hat{x}_t - x^*_{T-1}} + \norm{\hat{x}_T - x^*_{T-1}}} \label{eq:thm1-triangle-ineq-2}\\
&\leq 2\norm{\nabla F^{(T-1)}(\hat{x}_T)} + \lambda \sum_{t=1}^{T-1}2^{t}\norm{\hat{x}_t - x^*_{T-1}} \label{eq:thm1-strong-convexity}\\
&\leq 2\norm{\nabla F^{(T-1)}(\hat{x}_T)} + 4\sum_{t=1}^{T-1} \sqrt{\lambda 2^t \delta_t} \label{eq:thm1-lem4}\\
&\leq 4\sqrt{H\delta_T} + 4\sum_{t=1}^{T-1} \sqrt{\lambda 2^t \delta_t} \label{eq:thm1-smoothness}\\
&\leq 4\sum_{t=1}^{T} \sqrt{\lambda 2^{t+1} \delta_t}. \label{eq:thm1-choice-of-T}
\end{align}
Above, \pref{eq:thm1-triangle-ineq-1} and \pref{eq:thm1-triangle-ineq-2} rely on the triangle inequality; \pref{eq:thm1-strong-convexity} follows from the $\parens{\lambda\sum_{t=1}^{T-1}2^{t}}$-strong convexity of $F^{(T-1)}$; \pref{eq:thm1-lem4} applies the third conclusion of \pref{lem:zeyuan_auxillary_lemma}; \pref{eq:thm1-smoothness} uses the fact that $F^{(t-1)}$ is $H + \lambda\sum_{t=1}^{T-1}2^{t} < H + \lambda2^T = H + \lambda 2^{\lfloor \log H/\lambda \rfloor} \leq 2H$-smooth; and finally \pref{eq:thm1-choice-of-T} uses that $H \leq \lambda 2^{T+1}$.

We chose $\cA(F^{(t-1)}, \hat{x}_{t-1})$ to be $\gltwo$ applied to $F^{(t-1)}$ initialized at $\hat{x}_{t-1}$ using $m/T$ stochastic gradients. Therefore, 
\begin{align}
\delta_t &\leq 
\frac{128H^2\mathbb{E}\norm{\hat{x}_{t-1}-x^*_{t-1}}^2}{2^{t-1}\lambda(m/T)^4} 
+ \frac{256H\varsquared}{2^{2t-2}\lambda^2(m/T)^3} 
+ \frac{16\varsquared}{2^{t-1}\lambda(m/T)}.
\intertext{Using part two of \pref{lem:zeyuan_auxillary_lemma}, for $t > 1$ we can bound $\mathbb{E}\norm{\hat{x}_{t-1}-x^*_{t-1}}^2 \leq \frac{\delta_{t-1}}{2^{t-1}\lambda}$, thus}
\delta_t &\leq 
\frac{128H^2\delta_{t-1}}{2^{2t-2}\lambda^2(m/T)^4} 
+ \frac{256H\varsquared}{2^{2t-2}\lambda^2(m/T)^3} 
+ \frac{16\varsquared}{2^{t-1}\lambda(m/T)}. \label{eq:thm1-delta-inequality}
\end{align}

We can therefore bound
\begin{align}
8\sum_{t=1}^{T} \sqrt{\lambda 2^{t-1} \delta_t}
&\leq 
8\sqrt{\frac{128H^2\norm{x_0-x^*}^2}{(m/T)^4} 
+ \frac{256H\varsquared}{\lambda(m/T)^3} 
+ \frac{16\varsquared}{(m/T)}} \label{eq:thm1-using-delta-inequality}\\
&\qquad+ 8\sum_{t=2}^{T} \sqrt{\frac{128H^2\delta_{t-1}}{2^{t-1}\lambda(m/T)^4} 
+ \frac{256H\varsquared}{2^{t-1}\lambda(m/T)^3} 
+ \frac{16\varsquared}{(m/T)}} \nonumber\\~\notag \\
&\leq 
8\sqrt{\frac{128H^2\norm{x_0-x^*}^2}{(m/T)^4}} 
+ 8\sqrt{\frac{256H\varsquared}{\lambda(m/T)^3}} 
+ 8\sqrt{\frac{16\varsquared}{(m/T)}} \label{eq:thm1-sqrt-inequality}\\
&\qquad+ 8\sum_{t=2}^{T} 
\sqrt{\frac{128H^2\delta_{t-1}}{2^{t-1}\lambda(m/T)^4}}
+ \sqrt{\frac{256H\varsquared}{2^{t-1}\lambda(m/T)^3}}
+ \sqrt{\frac{16\varsquared}{(m/T)}} \nonumber\\~\notag\\
&= 
\frac{64\sqrt{2}H\norm{x_0-x^*}T^2}{m^2} 
+ \frac{128\sqrt{H}\sigma T^{3/2}}{\sqrt{\lambda}m^{3/2}}\sum_{t=1}^{T}\frac{1}{\sqrt{2^{t-1}}} \\
&\qquad+ \frac{32\sigma T^{3/2}}{\sqrt{m}}
+ \frac{128HT^2}{m^2}\sum_{t=2}^{T} 
\sqrt{\frac{\delta_{t-1}}{2^{t-2}\lambda}} \nonumber\\ ~\nonumber \\
&\leq 
\frac{64\sqrt{2}H\norm{x_0-x^*}T^2}{m^2} 
+ \frac{512\sqrt{H}\sigma T^{3/2}}{\sqrt{\lambda}m^{3/2}} \label{eq:thm1-sum-and-final-term}\\
&\qquad+ \frac{32\sigma T^{3/2}}{\sqrt{m}}
+ \frac{128HT^2}{m^2}\sum_{t=1}^{T} 
\sqrt{\frac{\delta_{t}}{2^{t-1}\lambda}} \nonumber\\ ~\nonumber \\
&\leq
\frac{64\sqrt{2}H\norm{x_0-x^*}T^2}{m^2} 
+ \frac{512\sqrt{H}\sigma T^{3/2}}{\sqrt{\lambda}m^{3/2}} \label{eq:thm1-multiply-by-2tothet}\\
&\qquad+ \frac{32\sigma T^{3/2}}{\sqrt{m}}
+ \frac{128HT^2}{\lambda m^2}\sum_{t=1}^{T} 
\sqrt{\lambda 2^{t-1}\delta_{t}}. \nonumber
\end{align}
Above, we arrive at \pref{eq:thm1-using-delta-inequality} by upper bounding each $\delta_t$ via \pref{eq:thm1-delta-inequality}; \pref{eq:thm1-sqrt-inequality} follows from the fact that for $a,b \geq 0$, $\sqrt{a+b} \leq \sqrt{a} + \sqrt{b}$; \pref{eq:thm1-sum-and-final-term} uses the fact that $\sum_{t=1}^{T}\frac{1}{\sqrt{2^{t-1}}} \leq 4$ and $\sqrt{\frac{\delta_{T}}{2^{T-1}\lambda}} \geq 0$; and finally, \pref{eq:thm1-multiply-by-2tothet} follows by multiplying each non-negative term in the sum by $2^{t-1}$. Rearranging inequality \pref{eq:thm1-multiply-by-2tothet} and combining with \pref{eq:thm1-choice-of-T} yields
\begin{equation}
\mathbb{E}\norm{\nabla F(\hat{x}_T)} 
\leq 
\parens{\frac{1}{1-\frac{16HT^2}{\lambda m^2}}}
\parens{
\frac{64\sqrt{2}H\norm{x_0-x^*}T^2}{m^2} 
+ \frac{512\sqrt{H}\sigma T^{3/2}}{\sqrt{\lambda}m^{3/2}}
+ \frac{32\sigma T^{3/2}}{\sqrt{m}}
}.\label{eq:thm1-gradient-bound-m}
\end{equation}
Choosing $m > 8T\sqrt{\frac{H}{\lambda}}$ ensures that the first term is at most $2$, and then solving for $m$ such that the second term is $O(\epsilon)$ completes the proof.
\end{proof}

\begin{lemma}\label{lem:closeness-of-optima}
For any $F$, define $\tilde{F}(x) = F(x) + \frac{\lambda}{2}\norm{x-x_0}$.
Then
\begin{enumerate}
\item $F \in \FRHz \implies \tilde{F} \in \dbF{H+\lambda}{\lambda}{R}$ and $\forall x \norm{\nabla F(x)} \leq 2\norm{\nabla \tilde{F}(x)} + \lambda R.$
\item $F \in \FDHz$ $\implies$  \\ \hfill $\tilde{F} \in \dbF{H+\lambda}{\lambda}{R=\sqrt{2\Delta/\lambda}}$ and $\forall x \norm{\nabla F(x)} \leq 2\norm{\nabla \tilde{F}(x)} + \sqrt{2\lambda \Delta}.$
\end{enumerate}
\end{lemma}
\begin{proof}
Let $\tilde{x}^* \in \argmin_x \tilde{F}(x)$. Since $\nabla \tilde{F}(x) = \nabla F(x) + \lambda(x-x_0)$,
\begin{align}
\norm{\nabla F(x)}
&\leq \norm{\nabla \tilde{F}(x)} + \lambda \norm{x - x_0} \\
&\leq \norm{\nabla \tilde{F}(x)} + \lambda \norm{x_0 - \tilde{x}^*} + \lambda \norm{x - \tilde{x}^*} \\
&\leq 2\norm{\nabla \tilde{F}(x)} + \lambda \norm{x_0 - \tilde{x}^*},\label{eq:lemcloseness-grads}
\end{align}
where we used the $\lambda$-strong convexity of $\tilde{F}$ for the last inequality. 
Similarly, $0 = \nabla \tilde{F}(\tilde{x}^*) = \nabla F(\tilde{x}^*) + \lambda(\tilde{x}^*-x_0)$. Therefore,
\begin{align}
\lambda\norm{x_0-\tilde{x}^*}^2 
&= \inner{\nabla F(\tilde{x}^*)}{x_0 - \tilde{x}^*} \\
&= \inner{\nabla F(\tilde{x}^*)}{x_0 - x^*} + \inner{\nabla F(\tilde{x}^*)}{x^* - \tilde{x}^*} \\
&\leq \inner{\nabla F(\tilde{x}^*)}{x_0 - x^*} \\
&= \inner{\lambda(x_0 - \tilde{x}^*)}{x_0 - x^*} \\
&\leq \lambda \norm{x_0 - \tilde{x}^*}\norm{x_0 - x^*}.
\end{align}
The first inequality follows from the convexity of $F$ and the second from the Cauchy-Schwarz inequality. When $F \in \FRHz$, then $\norm{x_0-\tilde{x}^*} \leq R$, which, combined with \pref{eq:lemcloseness-grads} proves the first claim.

Alternatively, when $F \in \FDHz$
\begin{equation}
F(x_0) = \tilde{F}(x_0) \geq \tilde{F}(\tilde{x}^*) = F(\tilde{x}^*) + \frac{\lambda}{2}\norm{x_0-\tilde{x}^*}^2.
\end{equation}
Rearranging, 
\begin{equation}
\norm{x_0-\tilde{x}^*} \leq \sqrt{\frac{2(F(x_0) - F(\tilde{x}^*))}{\lambda}} \leq \sqrt{\frac{2(F(x_0) - F(x^*))}{\lambda}} \leq \sqrt{\frac{2\Delta}{\lambda}}.
\end{equation}
This, combined with \pref{eq:lemcloseness-grads}, completes the proof.
\end{proof}

\accsgdthreeboundeddomain*

\begin{proof}
We use Algorithm 1 with $\gltwo$ as its subroutine to optimize $\tilde{F}(x) = F(x) + \frac{\lambda}{2}\norm{x-x_0}^2$. Our choice of $\lambda = \frac{256H\log^2(m^2)}{m^2} \leq O(H)$ ensures that $\tilde{F}$ is $H+\lambda \leq O(H)$-smooth and $\lambda$-strongly convex; that $\frac{16(H+\lambda)\log^2\parens{\frac{H+\lambda}{\lambda}}}{\lambda m^2} \leq \frac{1}{2}$; and finally that $\frac{H}{\lambda} \leq m^2$.
Therefore, by \pref{thm:acc-sgd3-strongly-convex}, in particular, \pref{eq:thm1-gradient-bound-m}, the output satisfies 
\begin{align}
\mathbb{E}\norm{\nabla \tilde{F}(\hat{x})} 
&\leq 
O\parens{
\frac{H\norm{x_0-\tilde{x}^*}\log^2\parens{H/\lambda}}{m^2} 
+ \frac{\sqrt{H}\sigma \log^{3/2}\parens{H/\lambda}}{\sqrt{\lambda}m^{3/2}}
+ \frac{\sigma \log^{3/2}\parens{H/\lambda}}{\sqrt{m}}
}\\
&\leq 
O\parens{
\frac{H\norm{x_0-\tilde{x}^*}\log^2\parens{m}}{m^2} 
+ \frac{\sigma \log^{3/2}\parens{m}}{\sqrt{m}}
},
\end{align} 
where $\tilde{x}^* = \argmin_x \tilde{F}(x)$. For $F \in \FRHz$, by part one of \pref{lem:closeness-of-optima}, $\norm{x_0-\tilde{x}^*} \leq R$ and
\begin{equation}
\mathbb{E}\norm{\nabla F(\hat{x})}
\leq
O\parens{
\frac{HR\log^2\parens{m}}{m^2} 
+ \frac{\sigma \log^{3/2}\parens{m}}{\sqrt{m}}
}.
\end{equation}
Solving for $m$ such that this expression is $O(\epsilon)$ completes the first part of the proof. For this $m$,
\begin{equation}
\lambda = \Theta\parens{\min\crl*{\frac{\epsilon}{R}, \frac{H\epsilon^4}{\sigma^4\log^4\parens{\sigma/\epsilon}}}}.
\end{equation}
For $F \in \FDHz$, by part two of \pref{lem:closeness-of-optima}, $\norm{x_0-\tilde{x}^*} \leq \sqrt{2\Delta/\lambda}$ and
\begin{equation}
\mathbb{E}\norm{\nabla F(\hat{x})}
\leq 
O\parens{
\frac{\sqrt{H\Delta}\log\parens{m}}{m} 
+ \frac{\sigma \log^{3/2}\parens{m}}{\sqrt{m}}}.
\end{equation}
Solving for $m$ such that this expression is $O(\epsilon)$ completes the the proof. For this $m$, 
\begin{equation}
\lambda = \Theta\parens{\min\crl*{\frac{\epsilon^2}{\Delta}, \frac{H\epsilon^4}{\sigma^4\log^4\parens{\sigma/\epsilon}}}}.
\end{equation}
\end{proof}

\section{Proofs from \pref{sec:sample-complexity}: Upper Bounds}\label{app:proofs-sample-complexity}

\begin{theorem}
\label{thm:erm_variance}
For any $F \in \FHL$ and any $\sample$ with the restriction that $f(x; z)$ is $\lambda$-strongly convex with respect to $x$ for all $z$, define the empirical risk minimizer via
\[
\hat{x}=\argmin_{x\in\bbR^{d}}\frac{1}{m}\sum_{t=1}^{m}f(x;z_t).
\]
Then the empirical risk minimizer enjoys the guarantee
\begin{equation}
\label{eq:erm_parameter}
\En\nrm*{\hat{x}-x^*}^{2}\leq{} \frac{4\sigma^{2}}{\lambda^{2}m}.
\end{equation}
\end{theorem}
\begin{proof}
Let $\wh{F}_m(x)=\frac{1}{m}\sum_{t=1}^{m}f(x;z_t)$ be the empirical objective. Since $f(x;z_t)$ is $\lambda$-strongly convex for each $z_t$, $\wh{F}_m$ is itself $\lambda$-strongly convex, and so we have
\[
\tri*{\grad{}\empF(x^{\star}),\hat{x}-x^*} + \frac{\lambda}{2}\nrm*{\hat{x}-x^*}^{2}\leq{}\empF(\hat{x})-\empF(x^*).
\]
Since, $\hat{x}$ is the empirical risk minimizer, we have $\empF(\hat{x})-\empF(x^*)\leq0$, and so, rearranging,
\[
\frac{\lambda}{2}\nrm*{\hat{x}-x^*}^{2}
\leq{} \tri*{\grad{}\empF(x^*),\hat{x}-x^*} 
\leq{} \nrm*{\grad{}\empF(x^*)}\nrm*{\hat{x}-x^*}.
\]
If $\hat{x}-x^* = 0$, then we are done. Otherwise, 
\[
\nrm*{\hat{x}-x^*}
\leq{} \frac{2}{\lambda}\nrm*{\grad{}\empF(x^*)}.
\]
Now square both sides and take the expectation, which gives
\[
\En\nrm*{\hat{x}-x^*}_{2}^{2}
\leq{} \frac{4}{\lambda^{2}}\En\nrm*{\grad{}\empF(x^*)}^{2}.
\]
The final result follows by observing that $\En\nrm*{\grad{}\empF(x^*)}^{2}\leq\frac{\sigma^{2}}{m}$.
\end{proof}

\samplecomplexitySC*
\begin{proof}
Consider the function $F^{(T)}(x) = F(x) + \lambda\sum_{t=1}^{T}2^{t-1}\norm{x-\hat{x}_t}^2$. Then
\begin{align}
\norm{\nabla F(\hat{x}_T)} 
&= \norm{\nabla F^{(T)}(\hat{x}_T) + \lambda\sum_{t=1}^{T}2^{t}(\hat{x}_t - \hat{x}_T)} \\
&\leq \norm{\nabla F^{(T)}(\hat{x}_T)} + \lambda\sum_{t=1}^{T-1}2^{t}\norm{\hat{x}_t - \hat{x}_T} \label{eq:thm2-triangle-ineq-1}\\
&\leq \norm{\nabla F^{(T)}(\hat{x}_T)} + \lambda\sum_{t=1}^{T-1}2^{t}\parens{\norm{\hat{x}_t - x^*_{T}} + \norm{\hat{x}_T - x^*_{T}}} \label{eq:thm2-triangle-ineq-2}\\
&\leq 2\norm{\nabla F^{(T)}(\hat{x}_T)} + \lambda \sum_{t=1}^{T-1}2^{t}\norm{\hat{x}_t - x^*_{T}} \label{eq:thm2-strong-convexity}\\
&\leq 6H\norm{\hat{x}_T - x^*_{T}} + \lambda \sum_{t=1}^{T-1}2^{t}\norm{\hat{x}_t - x^*_{T}} \label{eq:thm2-smoothness} \\
&\leq 12\lambda \sum_{t=1}^{T}2^{t}\norm{\hat{x}_t - x^*_{T}}. \label{eq:thm2-grad-bound}
\end{align}
Above, \pref{eq:thm2-triangle-ineq-1} and \pref{eq:thm2-triangle-ineq-2} rely on the triangle inequality; \pref{eq:thm2-strong-convexity} follows from the $\parens{\lambda\sum_{t=1}^{T}2^{t}}$-strong convexity of $F^{(T)}$; \pref{eq:thm2-smoothness} uses the fact that $F^{(T)}$ is $H + \lambda\sum_{t=1}^{T}2^{t} < H + \lambda2^{T+1} = H + 2\lambda 2^{\lfloor \log H/\lambda \rfloor} \leq 3H$-smooth.

Define $P_k = \sum_{t=1}^k 2^t\norm{\hat{x}_t - x^*_{k}}$ for $1 \leq k \leq T$ with $P_0 = 0$. Note that our upper bound \pref{eq:thm2-grad-bound} is equal to $12\lambda P_T = 12\lambda\sum_{k=1}^{T}P_k - P_{k-1}$, so we will estimate the terms of this sum.
\begin{align}
P_k - P_{k-1} 
&= 2^k\norm{\hat{x}_k - x^*_k} + \sum_{t=1}^{k-1} 2^t\parens{\norm{\hat{x}_t - x^*_{k}} - \norm{\hat{x}_t - x^*_{k-1}}} \\
&\leq 2^k\norm{\hat{x}_k - x^*_k} + \sum_{t=1}^{k-1} 2^t\norm{x^*_{k} - x^*_{k-1}} \label{eq:thm2-reverse-triangle-ineq}\\
&\leq 2^k\parens{\norm{\hat{x}_k - x^*_k} + \norm{x^*_{k} - x^*_{k-1}}} \\
&\leq 2^k\parens{2\norm{\hat{x}_k - x^*_k} + \norm{\hat{x}_{k} - x^*_{k-1}}}. \label{eq:thm2-revised-upper-bound}
\end{align}
Above, we used the reverse triangle inequality to derive \pref{eq:thm2-reverse-triangle-ineq}. By optimality of $x^*_{k-1}$ and $x^*_{k}$,
\begin{equation}
\nrm*{\hat{x}_k-x^*_{k-1}}^{2} - \nrm*{\hat{x}_k-x^*_{k}}^{2}
= 
\frac{F^{\expi{k}}(x^*_{k-1}) - F^{\expi{k}}(x^*_{k})
+ F^{\expi{k-1}}(x^*_{k}) - F^{\expi{k-1}}(x^*_{k-1})}{2^{k-1}\lambda} \geq{} 0.
\end{equation}
Thus $\norm{\hat{x}_k - x^*_k} \leq \norm{\hat{x}_{k} - x^*_{k-1}}$ and, combining \pref{eq:thm2-grad-bound} and \pref{eq:thm2-revised-upper-bound} yields
\begin{equation}
\norm{\nabla F(\hat{x}_T)} \leq 36\lambda\sum_{t=1}^T 2^t\norm{\hat{x}_t - x^*_{t-1}}.
\end{equation}
Since $\hat{x}_t$ is the output of ERM on the $2^{t-1}\lambda$-strongly convex function $F^{t-1}$ using $m/T$ samples, by \pref{thm:erm_variance}, $\mathbb{E}\norm{\hat{x}_t - x^*_{t-1}} \leq \frac{2\sigma\sqrt{T}}{2^{t-1}\lambda \sqrt{m}}$ and
\begin{align}
\mathbb{E}\norm{\nabla F(\hat{x}_T)} 
&\leq 36\lambda\sum_{t=1}^T 2^t\mathbb{E}\norm{\hat{x}_t - x^*_{t-1}} \\
&\leq 36\lambda\sum_{t=1}^T 2^t \frac{\sigma\sqrt{T}}{2^{t-2}\lambda \sqrt{m}} \\
&= \frac{144\sigma T^{3/2}}{\sqrt{m}}. \label{eq:thm2-gradient-bound-m}
\end{align}
Solving for $m$ such that the expression is less than $\epsilon$ completes the proof.
\end{proof}

\samplecomplexitybounded*
\begin{proof}
The objective function $\tilde{F}(x) = F(x) + \frac{\lambda}{2}\norm{x-x_0}^2$ is $(H+\lambda)$-smooth and $\lambda$-strongly convex. Thus by \pref{thm:sample-complexity-strong-convexity}, in particular \pref{eq:thm2-gradient-bound-m}, the output of the algorithm satisfies
\begin{equation}
\mathbb{E}\norm{\nabla \tilde{F}(\hat{x})} 
\leq \frac{144\sigma \log^{3/2}\parens{\frac{H+\lambda}{\lambda}}}{\sqrt{m}}.
\end{equation}
For $F \in \FRHz$, with $\lambda = \Theta(\epsilon/R)$ and $m = \Omega\parens{\frac{\sigma^2}{\epsilon^2}\log^3\parens{\frac{HR}{\epsilon}}}$ and using part one of \pref{lem:closeness-of-optima} we conclude
\begin{align}
\mathbb{E}\norm{\nabla F(\hat{x})}
&\leq O(\epsilon + \lambda R) \leq O(\epsilon),
\end{align}
which completes the first part of the proof.

Similarly, for $F \in \FDHz$, with $\lambda = \Theta(\epsilon^2/\Delta)$ and $m = \Omega\parens{\frac{\sigma^2}{\epsilon^2}\log^3\parens{\frac{\sqrt{H\Delta}}{\epsilon}}}$, by part two of \pref{lem:closeness-of-optima} we conclude
\begin{equation}
\mathbb{E}\norm{\nabla F(\hat{x})} 
\leq O\parens{\epsilon + \sqrt{\lambda\Delta}} \leq O(\epsilon),
\end{equation}
which completes the proof.
\end{proof}

\samplecomplexitynonsmooth*
\begin{proof}
Our algorithm involves calculating the ERM on several independent samples, evaluating the gradient norm at these ERMs on a held-out sample, and returning the point with the smallest gradient norm.

Let $\nabla_- F(x)$ denote the left-derivative of $F$ at $x$, and let $\nabla_+ F(x)$ denote the right-derivative. Since $F$ is bounded from below, $\lim_{x\to -\infty} \nabla_- F(x) \leq 0$ and $\lim_{x\to \infty} \nabla_+ F(x) \geq 0$, thus there exists at least one $\epsilon$-stationary point for $F$. Consequently, there is a unique $a \in \mathbb{R}\cup\crl{-\infty}$ for which $\nabla_+ F(a) \geq -\epsilon$ and $\forall x < a\ \nabla_+ F(x) < -\epsilon$. The point $a$ is the left-most $\epsilon$-stationary point. It is possible that $a = -\infty$, in which case there are no $x < a$. Similarly, there is a unique $b \in \mathbb{R}\cup\crl{\infty}$ for which $\nabla_- F(b) \leq \epsilon$ and $\forall x > b\ \nabla_- F(x) > \epsilon$. The point $b$ is the right-most $\epsilon$-stationary point. It is possible that $b = \infty$, in which case there are no $x > b$.

By convexity, $\forall x<y$ $\nabla_- F(x) \leq \nabla_+ F(x) \leq \nabla_- F(y) \leq \nabla_+ F(y)$. Therefore, $x < a \implies \inf_{g\in\partial F(x)} ~ \abs*{g} \geq \abs{\nabla_+ F(x)} > \epsilon$ and $x > b \implies \inf_{g\in\partial F(x)} \abs{g} \geq \abs{\nabla_- F(x)} > \epsilon$. Therefore, $[a,b] \equiv \crl{x:\inf_{g\in\partial F(x)} \abs{g} \leq \epsilon}$. Consequently, all that we need to show is that our algorithm returns a point within the interval $[a,b]$.

Let $\hat{F}(x) = \frac{1}{m}\sum_{i=1}^m f(x;z_i)$ be the empirical objective function and let $\hat{x}$ be any minimizer of $\hat{F}$. Consider first the case that $a > -\infty$, we will argue that $\hat{x} \geq a$. Observe that if $\nabla_- \hat{F}(a) < 0$, then since $\hat{F}$ is convex, it is decreasing on $[-\infty,a]$ and thus $\hat{x} \geq a$. Since $a > -\infty$, $\nabla_- F(a) \leq -\epsilon$, so the value $\nabla_- \hat{F}(a) = \frac{1}{m}\sum_{i=1}^m \nabla_- f(a;z_i)$ is the sum of i.i.d.~random variables that have mean $\nabla_- F(a) \leq -\epsilon$ and variance $\sigma^2$.
By Chebyshev's inequality, the random variable $\nabla_- \hat{F}(a)$ will not deviate too far from its mean:
\begin{equation}
    \mathbb{P}\brk*{\nabla_- \hat{F}(a) \geq 0} \leq \frac{\sigma^2}{m\epsilon^2}.
\end{equation}
Similarly, 
\begin{equation}
    \mathbb{P}\brk*{\nabla_+ \hat{F}(b) \leq 0} \leq \frac{\sigma^2}{m\epsilon^2}.
\end{equation}
Therefore, with probability at least $1-\frac{2\sigma^2}{m\epsilon^2}$, the minimum of $\hat{F}$ lies in the range $[a,b]$ and thus the ERM $\hat{x}$ is an $\epsilon$-stationary point of $F$.

Consider calculating $k$ ERMs $\hat{x}_1,\dots,\hat{x}_k$ on $k$ independent samples of size $m$. Then with probability at least $1 - \parens{\frac{2\sigma^2}{m\epsilon^2}}^k$, at least one of these points is an $\epsilon$-stationary point of $F$.

Now, suppose we have $km$ additional heldout samples which constitute an empirical objective $\hat{F}$. Since the ERMs $\hat{x}_i$ are independent of these samples, 
\begin{equation}
    \mathbb{E}\brk*{\max_{i\in[k]} \norm{\nabla \hat{F}(\hat{x}_i) - \nabla F(\hat{x}_i)}^2} \leq \sum_{i=1}^k\mathbb{E}\brk*{\norm{\nabla \hat{F}(\hat{x}_i) - \nabla F(\hat{x}_i)}^2} \leq \frac{k\sigma^2}{km} = \frac{\sigma^2}{m}.
\end{equation}
Condition on the event that at least one of the ERMs is an $\epsilon$-stationary point of $F$ and denote one of those ERMs as $\hat{x}_{i^*}$. Denote this event $E$. Let $\hat{i} \in \argmin_{i} \norm{\nabla \hat{F}(\hat{x}_i)}$ where we abuse notation and say $\norm{\nabla \hat{F}(\hat{x}_i)} := \inf_{g\in\partial \hat{F}(\hat{x}_i)}\abs{g}$ for cases where $\hat{F}$ is not differentiable at $\hat{x}_i$. Then
\begin{align}
\mathbb{E}\brk*{\norm{\nabla F(\hat{x}_{\hat{i}})}\middle| E} 
&\leq \mathbb{E}\brk*{\norm{\nabla \hat{F}(\hat{x}_{\hat{i}})}\middle| E} + \mathbb{E}\brk*{\max_{i\in[k]} \norm{\nabla \hat{F}(\hat{x}_i) - \nabla F(\hat{x}_i)}\middle| E} 
\\
&\leq \mathbb{E}\brk*{\norm{\nabla \hat{F}(\hat{x}_{\hat{i}})}\middle| E} + \sqrt{\frac{\sigma^2}{m}} \\
&\leq \mathbb{E}\brk*{\norm{\nabla \hat{F}(\hat{x}_{i^*})}\middle| E} + \sqrt{\frac{\sigma^2}{m}} \\
&\leq \mathbb{E}\brk*{\norm{\nabla F(\hat{x}_{i^*})}\middle| E} + \mathbb{E}\brk*{\max_{i\in[k]} \norm{\nabla \hat{F}(\hat{x}_i) - \nabla F(\hat{x}_i)}\middle| E} + \sqrt{\frac{\sigma^2}{m}} \\
&\leq \epsilon + 2\sqrt{\frac{\sigma^2}{m}}.
\end{align}
The event that one of the ERMs is an $\epsilon$-stationary point happens with probability at least $1-\parens{\frac{2\sigma^2}{m\epsilon^2}}^k$. Choosing $m = \Omega\parens{\frac{\sigma^2}{\epsilon^2}}$ and $k = \Omega\parens{\log\frac{L}{\epsilon}}$ ensures $1-\parens{\frac{2\sigma^2}{m\epsilon^2}}^k \geq 1 - \frac{\epsilon}{L}$. Therefore, 
\begin{align}
\mathbb{E}\brk*{\norm{\nabla F(\hat{x}_{\hat{i}})}}
&= \mathbb{P}\brk{E}\mathbb{E}\brk*{\norm{\nabla F(\hat{x}_{\hat{i}})}\middle| E} + \mathbb{P}\brk{E^c}\mathbb{E}\brk*{\norm{\nabla F(\hat{x}_{\hat{i}})}\middle| E^c} \\
&\leq \parens{1 - \frac{\epsilon}{L}}\parens{\epsilon + 2\sqrt{\frac{\sigma^2}{m}}} + \parens{\frac{\epsilon}{L}} (L) 
\\
&\leq O(\epsilon).
\end{align}
This entire algorithm required $O(km) = O\parens{\frac{\sigma^2\log\parens{\frac{L}{\epsilon}}}{\epsilon^2}}$ samples in total, completing the proof. 
\end{proof}

\section{Proofs of the Lower Bounds} \label{app:proofs-lower-bounds}
\statlowerbound*
\begin{proof}
For a constant $b \in \mathbb{R}$ to be chosen later, let 
\begin{equation}
    f(x; z) = \sigma\inner{x}{z} + \frac{b}{2}\norm{x}^2.
\end{equation}
The distribution $\cD$ of the random variable $z$ is the uniform distribution over $\crl{z_1,\dots,z_m}$ where the vectors $z_i \in \mathbb{R}^d$ are orthonormal ($d \geq m$). Therefore,
\begin{equation}
    F(x) = \mathbb{E}\brk{f(x; z)} = \sigma \inner{x}{\frac{1}{m}\sum_{i=1}^m z_i} + \frac{b}{2}\norm{x}^2.
\end{equation}
This function is clearly convex, $b$-smooth, and attains its unique minimum at $x^* = -\frac{\sigma}{bm}\sum_{i=1}^m z_i$ which has norm $\norm{x^*}^2 = \frac{\sigma^2}{b^2m}$, so choosing $b \geq \frac{\sigma}{R\sqrt{m}}$ ensures $
\norm{x^*}^2 \leq R^2$. Furthermore, $F(0) - F(x^*) = \frac{\sigma^2}{2bm}$, so choosing $b \geq \frac{\sigma^2}{2\Delta m}$ ensures $F(0) - F(x^*) \leq \Delta$. Choosing $b = \max\crl{\frac{\sigma}{R\sqrt{m}}, \frac{\sigma^2}{2\Delta m}}$ ensures both simultaneously. Finally, $\mathbb{E}\norm{\nabla f(x; z) - \nabla F(x)} = \frac{1}{m}\sum_{i=1}^m\norm{\sigma z_i - \frac{\sigma}{m}\sum_{j=1}^m z_j}^2 = \sigma^2\parens{1-\frac{1}{m}} \leq \sigma^2$.

Therefore, $F \in \FRHz \cap \FDHz$ and $f,\cD$ properly define a $\sfo$.

Suppose, for now, that $x$ is a point such that $\inner{x}{v_i} \geq -\frac{\sigma}{8bm}$ for all $i \geq m/2$. Then
\begin{align}
\norm{\nabla F(x)}^2 
&= \frac{\sigma^2}{m} + b^2\norm{x}^2 + \frac{2\sigma b}{m}\sum_{i=1}^m \inner{x}{v_i} \\
&\geq \frac{\sigma^2}{m} + b^2 \sum_{i<m/2} \inner{x}{v_i}^2 + \frac{2\sigma b}{m}\sum_{i<m/2} \inner{x}{v_i} - \frac{\sigma^2}{4m} \\
&\geq \frac{3\sigma^2}{4m} + \min_{y\in\mathbb{R}} \frac{b^2m}{2} y^2 + \sigma b y \\
&= \frac{\sigma^2}{4m}.
\end{align}
Therefore, for all such vectors $x$, $\norm{\nabla F(x)} \geq \frac{\sigma}{2\sqrt{m}}$. This holds for any $b \geq 0$ and set $\crl{z_1,\dots,z_m}$. From here, we will argue that any randomized algorithm with access to less than $m/2$ samples from $\cD$ is likely to output such an $x$. We consider a random function instance determined by drawing the orthonormal set $\crl{z_1,\dots,z_m}$ uniformly at random from the set of orthonormal vectors in $\mathbb{R}^d$. We will argue that with moderate probability over the randomness in the algorithm and in the draw of $z_1,\dots,z_m$, the output of the algorithm has small inner product with $z_{m/2},\dots,z_m$. This approach closely resembles previous work \cite[Lemma 7]{woodworth16tight}.

Less than $m/2$ samples fix less than $m/2$ of the vectors $z_i$; assume w.l.o.g.~that the algorithm's sample $S = \crl{z_1,\dots,z_{m/2-1}}$. The vectors $z_i$ are a uniformly random orthonormal set, therefore for any $i \geq m/2$, $z_i | S$ is distributed uniformly on the $(d-m/2+1)$-dimensional unit sphere in the subspace orthogonal to $\textrm{span}\parens{z_1,\dots,z_{m/2-1}}$. Let $\hat{x}$ be the output of any randomized algorithm whose input is $S$. If $\norm{\hat{x}} \geq \frac{2\sigma}{b\sqrt{m}}$ then it is easily confirmed that $\norm{\nabla F(\hat{x})} \geq \frac{\sigma^2}{m}$. Otherwise, we analyze
\begin{align}
\mathbb{P}\left[\inner{\hat{x}}{v_i} < -\frac{\sigma}{8bm}\ \middle|\ S, \hat{x}\right] 
&\leq \mathbb{P}\left[\abs{\inner{\hat{x}}{v_i}} \geq \frac{\sigma}{8bm}\ \middle|\ S, \hat{x}\right] \\
&\leq \mathbb{P}\left[\frac{2\sigma}{b\sqrt{m}}\left|\inner{\frac{\hat{x}}{\norm{\hat{x}}}}{v_i}\right| \geq \frac{\sigma}{8bm}\ \middle|\ S, \hat{x}\right] \\
&= \mathbb{P}\left[\left|\inner{\frac{\hat{x}}{\norm{\hat{x}}}}{v_i}\right| \geq \frac{1}{16\sqrt{m}}\ \middle|\ S, \hat{x}\right].
\end{align}
This probability only increases if we assume that $\hat{x}$ is orthogonal to $\textrm{span}(z_1,\dots,z_{m/2-1})$, in which case we are considering the inner product between a fixed unit vector and a uniformly random unit vector. The probability of the inner product being large is proportional to the surface area of the ``cap'' of a unit sphere in $(d-m/2+1)$-dimensions lying above and below circles of radius $\sqrt{1-\frac{1}{256m}}$. These end caps, in total, have surface area less than that of a sphere with that same radius. Therefore,
\begin{align}
\mathbb{P}\left[\inner{\hat{x}}{v_i} < -\frac{\sigma}{8bm}\ \middle|\ S, \hat{x}\right] 
&\leq \sqrt{\prn*{1-\frac{1}{256m}}^{d-\frac{m}{2}}} \\
&= \parens{1 - \dfrac{\frac{d}{512m}-\frac{1}{1024}}{\frac{d}{2}-\frac{m}{4}}}^{\frac{d}{2}-\frac{m}{4}} \\
&\leq \exp\parens{\frac{1}{1024} - \frac{d}{512m}}.
\end{align}
This did not require anything but the norm of $\hat{x}$ being small, so for $d \geq \frac{m}{2} + 512m\log(2m)$, this ensures that 
\begin{equation}
\mathbb{P}\left[\inner{\hat{x}}{v_i} < -\frac{\sigma}{8bm}\ \middle|\ S, \norm{\hat{x}} < \frac{2\sigma}{b\sqrt{m}}\right] \leq \frac{1}{2m}.
\end{equation}
A union bound ensures that either $\norm{\hat{x}} \geq \frac{2\sigma}{b\sqrt{m}}$ or $\inner{\hat{x}}{v_i} \geq -\frac{\sigma}{8bm}$ for all $i \geq m/2$ with probability at least $1/2$ over the randomness in the algorithm and draw of $z_1,\dots,z_m$, and consequently, that $\mathbb{E}_{\hat{x}}\norm{\nabla F(\hat{x})}^2 \geq \frac{\sigma^2}{8m}$. Setting $m = \left\lfloor\frac{\sigma^2}{8\epsilon^2}\right\rfloor$ ensures this is at least $\epsilon$. For this $m$, $b = \max\crl{\frac{\sigma}{R\sqrt{m}}, \frac{\sigma^2}{2\Delta m}} \leq \frac{4\epsilon}{R} + \frac{4\epsilon^2}{\Delta}$ which must be less than $H$, consequently, this lower bound applies for $\epsilon \leq \min\crl{\frac{HR}{8}, \sqrt{\frac{H\Delta}{8}}}$.
\end{proof}

\begin{restatable}{theorem}{ohadLB}\label{thm:local-LB-log-term}
    For any $H,R,\sigma>0$ and any $\epsilon\in (0,\sigma/2)$, 
    there exists a $F:\bbR\rightarrow\bbR \in \FRHz$ and a $\sfo$, such that for any 
    algorithm interacting with the stochastic first-order oracle, and returning an 
    $\epsilon$-approximate stationary point with 
    some fixed constant probability, the expected number of queries is at least
    $\Omega\left(\frac{\sigma^2}{\epsilon^2}\cdot 
    \log\left(\frac{HR}{\epsilon}\right)\right)$. Moreover, a similar lower 
    bound of 
    $\Omega\left(\frac{\sigma^2}{\epsilon^2}\cdot 
    \log\left(\frac{H\Delta}{\epsilon^2}\right)\right)$ holds if the radius 
    constraint $R$ 
    is replaced 
    by a 
    suboptimality constraint $\Delta$. 
\end{restatable}
\begin{proof}
We prove the lower bound by reduction from the noisy binary search (NBS) 
problem: In 
this classical problem, we have $N$ sorted elements $\{a_1,\ldots,a_N\}$, and 
we wish to insert a new 
element $e$ using only queries of the form ``is $e>a_j$?'' for 
some $j$. Rather than 
getting the true answer, an independent coin is flipped and we get the correct 
answer only with probability $\frac{1}{2}+p$ for some fixed parameter $p$. 
Moreover, let $j^*$ be the unique index such that 
$a_{j^*}<e<a_{j^*+1}$\footnote{This 
is w.l.o.g., since if $e<a_1$ or $e>a_N$, we can just add two dummy elements 
smaller and larger than all other elements and $e$, increasing $N$ by at most 
$2$, hence not affecting the lower bound.}. It 
is well-known (see for example \cite{feige1994computing, karp2007noisy}) that in order to 
identify $j^*$ with any fixed constant probability, at least
$\Omega(\log(N)/p^2)$ queries are required.

Let us first consider the case where the radius constraint $R$ is fixed. We 
will construct a convex stochastic optimization problem with the given 
parameters, such that if there is 
an algorithm solving it (with constant probability) after $T$ local stochastic 
oracle queries, 
then it can be used to solve an NBS problem (with the same probability) using 
$2T$ queries, where $p=\epsilon/\sigma$ and\footnote{For simplicity we 
assume 
that $HR/4\epsilon$ 
is a whole number -- otherwise, it can be rounded and this will only affect 
constant factors in the lower bound.}
$N=HR/4\epsilon$. Employing the lower bound above for NBS, this 
immediately implies the $\Omega\left(\frac{\sigma^2}{\epsilon^2}\cdot 
\log\left(\frac{HR}{\epsilon}\right)\right)$ lower bound in our theorem.

To describe the reduction, let us first restate the NBS problem in a slightly 
different manner. For a fixed query budget $T$, let $Z$ be a $T\times N$ 
matrix, with entries in $\{-1,+1\}$ drawn independently according to the 
following distribution:
\begin{align*}
\Pr(Z_{t,j}=1)=\begin{cases} \frac{1}{2}-p&j\leq j^*\\
\frac{1}{2}+p&j> j*
\end{cases}~.
\end{align*}
Each $Z_{t,j}$ can be considered as the noisy answer provided in the NBS 
problem to the $t$-th query, of the form ``is $e>a_j$'' (where $-1$ corresponds 
to 
``true'' and $1$ corresponds to ``false''). Thus, an algorithm for the NBS 
problem can be seen as an algorithm which can query $T$ entries from the matrix 
$Z$ (one query from each row), and needs to find $j^*$ based on this 
information. Moreover, it is easy to see that the NBS lower bound also holds 
for an algorithm which can query \emph{any} $T$ entries from the 
matrix: Since the entries are independent, this does not provide additional 
information, and can only ``waste'' queries if the algorithm queries the same 
entry twice. 

We now turn to the reduction. Given an NBS problem on $N=HR/4\epsilon$ elements 
with $p=\epsilon/\sigma$ and a randomly-drawn matrix 
$Z$, we first divide the interval $[0,R]$ into $N$ equal sub-intervals of 
length $R/N$ each, and w.l.o.g. identify each element $a_j$ with the smallest 
point in the interval. Then, for every (statistically independent) row $Z_t$ 
of $Z$, we define a function 
$f(x,Z_t)$ on $\mathbb{R}$ by $f(0,Z_t)=0$, and the rest 
is defined via its derivative as follows:
\begin{align*}
f'(x,Z_t) = \begin{cases}-2\epsilon&x< 0\\2\epsilon&x\geq R\\
\frac{x-a_j}{R/N}\sigma Z_{t,j+1}+\left(1-\frac{x-a_j}{R/N}\right)\sigma 
Z_{t,j}&x\in [a_{j},a_{j+1})~\text{for some $j<N$}\end{cases}~.
\end{align*}
Note that by construction $\frac{x-a_j}{R/N}\in [0,1]$ and $Z_{t,j}\in 
\{-1,+1\}$, so $|f'(x,Z_t)|\leq \max\{2\epsilon,\sigma\}\leq \sigma$. Moreover, since 
the expected value of $\sigma Z_{t,j}$ is $\sigma\cdot (-2p)=-2\epsilon$ if $j\leq 
j^*$,
and $\sigma\cdot 2p=2\epsilon$ if $j>j^*$, it is easily verified that
\begin{align*}
\mathbb{E}_{Z_t}[f'(x,Z_t)]=\begin{cases}-2\epsilon & x<a_{j^*}\\
2\epsilon& x\geq a_{j^*+1}\\-2\epsilon+4\epsilon \frac{x-a_j}{R/N}&x\in 
[a_{j^*},a_{j^*+1})\end{cases}~.
\end{align*}
Noting that $4\epsilon\frac{x-a_j}{R/N} = H(x-a_j)\in [0,4\epsilon]$ in the 
above, we get that $F(x):=\mathbb{E}_{Z_t}[f(x,Z_t)]$ is a convex function 
with 
$H$-Lipschitz gradients, with a unique minimum at some $x:|x|<R$, and with 
$|F'(x)|\leq \epsilon$ only when $x\in [a_{j^*},a_{j^*+1})$. Overall, we get a 
valid convex stochastic optimization problem (with parameters $H,R,\sigma$ as 
required), such that if we can identify $x$ such that
$|F'(x)|\leq\epsilon$, then we 
can uniquely identify $j^*$. Moreover, given an algorithm to the optimization 
problem, we can simulate a query to a local stochastic oracle (specifying an 
iterate $t$ and a point $x$) by returning $f'(x,Z_t)$ as defined above, which 
requires querying at most $2$ entries $Z_{t,j}$ and $Z_{t,j+1}$ from the matrix 
$Z$. So, given an oracle query budget $T$ to the stochastic problem, we can 
simulate it with at most $2T$ queries to the matrix $Z$ in the NBS problem. 

To complete the proof of the theorem, it remains to handle the case where there 
is a suboptimality constraint $\Delta$ rather than a radius constraint $R$. To 
that end, we simply use the same construction as above, with 
$R=\frac{\Delta}{2\epsilon}$. Since the derivative of $F$ has magnitude at 
most $2\epsilon$, and its global minimum satisfies $|x^*|\leq R$, it follows 
that $F(0)-F(x^*)\leq 2\epsilon R = \Delta$. Plugging in 
$R=\frac{\Delta}{2\epsilon}$ in the lower bound, the result follows.
\end{proof}

\folowerbound*
\begin{proof}
By \pref{thm:local-LB-log-term}, $\Omega(\sigma^2/\epsilon^2)\log\parens{HR/\epsilon}$ and $\Omega(\sigma^2/\epsilon^2)\log\parens{H\Delta/\epsilon^2}$ oracle calls (samples) are needed to find an $\epsilon$-stationary point. Furthermore, a deterministic first-order oracle is a special case of a stochastic first-order oracle (corresponding to the case $\sigma = 0$). Therefore, lower bounds for deterministic first-order optimization apply also to stochastic first-order optimization. Therefore, the lower bound of \cite[Theorem 1]{carmon2017lower2} completes the proof.
\end{proof}

\end{document}